\documentclass[nohyperref]{article}

\pdfoutput=1

\usepackage{microtype}
\usepackage{graphicx}
\usepackage{caption}
\usepackage{subcaption}
\usepackage{booktabs} 

\usepackage{hyperref}

\usepackage[accepted]{icml2023}

\usepackage{amsmath}
\usepackage{amssymb}
\usepackage{mathtools}
\usepackage{amsthm}

\usepackage[capitalize,noabbrev]{cleveref}

\theoremstyle{plain}
\newtheorem{theorem}{Theorem}[section]

\newtheorem{lemma}[theorem]{Lemma}

\theoremstyle{definition}

\theoremstyle{remark}

\newcommand{\bs}{\boldsymbol{s}}
\newcommand{\ba}{\boldsymbol{a}}
\newcommand{\btau}{\boldsymbol{\tau}}
\usepackage{xspace}

\newcommand{\alias}{AdaptDiffuser\xspace}

\usepackage[textsize=tiny]{todonotes}

\icmltitlerunning{AdaptDiffuser: Diffusion Models as Adaptive Self-evolving Planners}

\begin{document}

\twocolumn[
\icmltitle{AdaptDiffuser: Diffusion Models as Adaptive Self-evolving Planners}



\icmlsetsymbol{equal}{*}

\begin{icmlauthorlist}
\icmlauthor{Zhixuan Liang}{hku}
\icmlauthor{Yao Mu}{hku}
\icmlauthor{Mingyu Ding}{hku,ucb}
\icmlauthor{Fei Ni}{tju}
\icmlauthor{Masayoshi Tomizuka}{ucb}
\icmlauthor{Ping Luo}{hku,shlab}

\end{icmlauthorlist}

\icmlaffiliation{hku}{Department of Computer Science, The University of Hong Kong, Hong Kong SAR}
\icmlaffiliation{ucb}{University of California, Berkeley, USA}
\icmlaffiliation{tju}{College of Intelligence and Computing, Tianjin University, Tianjin, China}
\icmlaffiliation{shlab}{Shanghai AI Laboratory, Shanghai, China}

\icmlcorrespondingauthor{Ping Luo}{pluo.lhi@gmail.com}

\icmlkeywords{Machine Learning, ICML}

\vskip 0.3in
]



\printAffiliationsAndNotice{}  

\begin{abstract}
Diffusion models have demonstrated their powerful generative capability in many tasks, with great potential to serve as a paradigm for offline reinforcement learning.
However, the quality of the diffusion model is limited by the insufficient diversity of training data, which hinders the performance of planning and the generalizability to new tasks.
This paper introduces \alias, an evolutionary planning method with diffusion that can self-evolve to improve the diffusion model hence a better planner, not only for seen tasks but can also adapt to unseen tasks.
\alias enables the generation of rich synthetic expert data for goal-conditioned tasks using guidance from reward gradients.
It then selects high-quality data via a discriminator to finetune the diffusion model, which improves the generalization ability to unseen tasks. 
Empirical experiments on two benchmark environments and two carefully designed unseen tasks in KUKA industrial robot arm and Maze2D environments demonstrate the effectiveness of \alias.
For example, \alias not only outperforms the previous art Diffuser~\cite{janner2022planning} by 20.8\% on Maze2D and 7.5\% on MuJoCo locomotion, but also adapts better to new tasks, e.g., KUKA pick-and-place, by 27.9\% without requiring additional expert data.
More visualization results and demo videos could be found on \href{https://adaptdiffuser.github.io/}{our project page}.
\end{abstract}

\begin{figure}[t]
\centering
\begin{subfigure}{0.49\textwidth}
  \centering
  \includegraphics[width=0.9\linewidth]{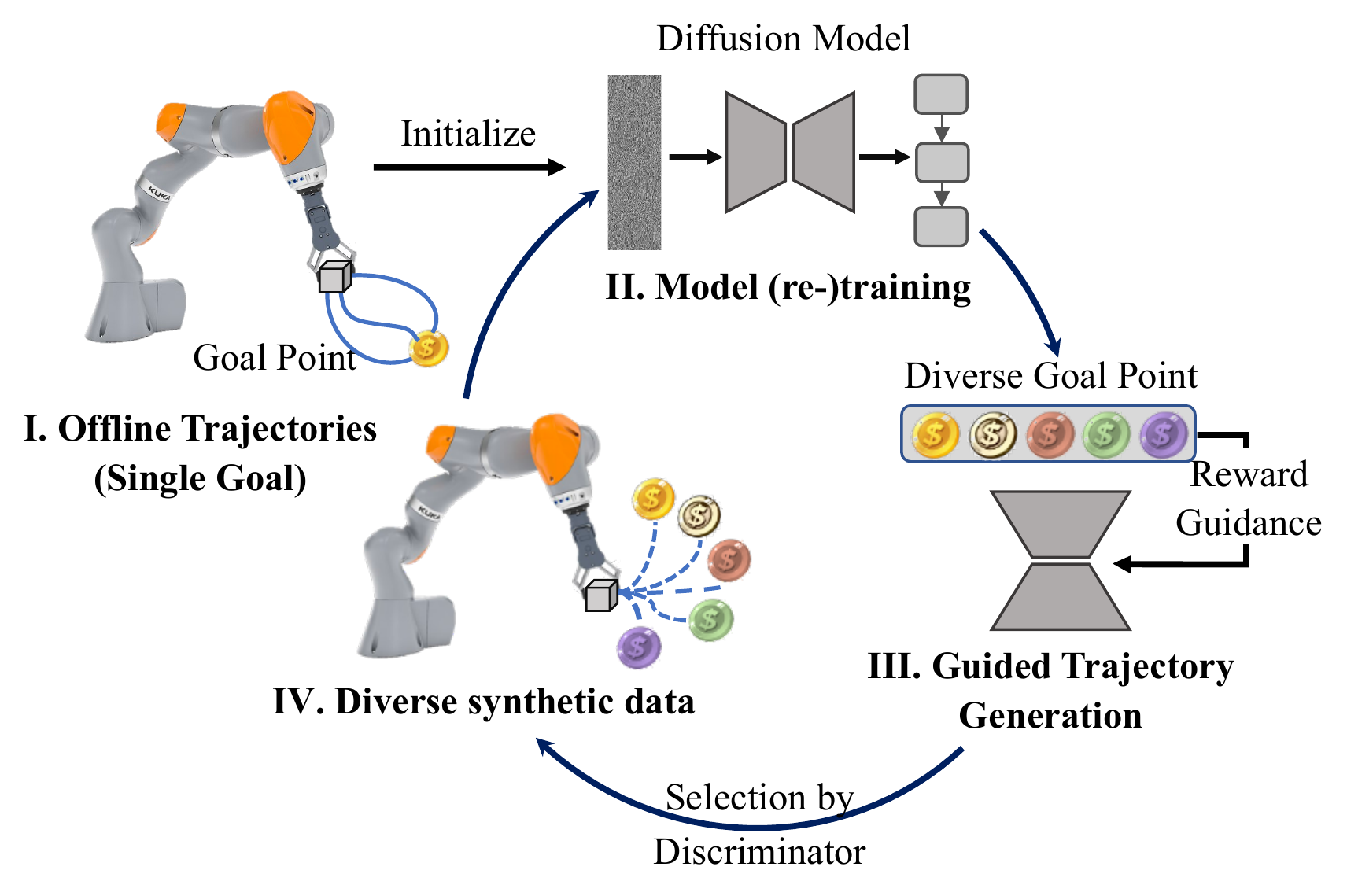}
  \vspace{-2pt}
  \caption{Illustration of \alias.}
  \label{fig:teaser-a}
\end{subfigure}
\begin{subfigure}{0.49\textwidth}
  \centering
  \vspace{8pt}
  \includegraphics[width=0.99\linewidth]{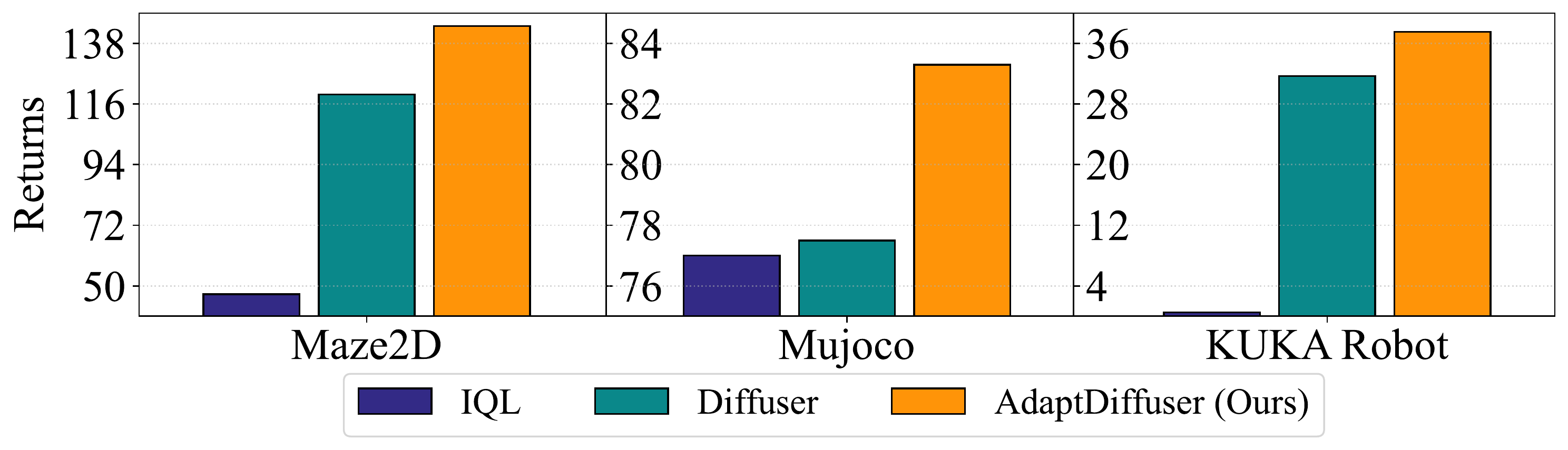}
  \vspace{-2pt}
  \caption{Performance comparisons on three benchmarks.}
  \label{fig:teaser-b}
\end{subfigure}
\caption{Overall framework and performance comparison of \alias. It enables diffusion models to generate rich synthetic expert data using guidance from reward gradients of either seen or unseen goal-conditioned tasks.
Then, it iteratively selects high-quality data via a discriminator to finetune the diffusion model for self-evolving, leading to improved performance on seen tasks and better generalizability to unseen tasks. }
\vspace{-10pt}
\label{fig:teaser}
\end{figure}

\section{Introduction}
Offline reinforcement learning (RL)~\cite{levine2020offline,prudencio2022survey} aims to learn policies from previously collected offline data without interacting with the live environment.
Traditional offline RL approaches require fitting value functions or computing policy gradients, which are challenging due to limited offline data~\cite{agarwal2020optimistic,kumar2020conservative,wu2019behavior,kidambi2020morel}. 
Recent advances in generative sequence modeling~\cite{chen2021decision, janner2021offline, janner2022planning} provide effective alternatives to conventional RL problems by modeling the joint distribution of sequences of states, actions, rewards and values.
For example, Decision Transformer~\cite{chen2021decision} casts offline RL as a form of conditional sequence modeling, which allows more efficient and stable learning without the need to train policies via traditional RL algorithms like temporal difference learning~\cite{sutton1988learning}.
By treating RL as a sequence modeling problem, it bypasses the need of bootstrapping for long-term credit assignment, avoiding one of the ``deadly triad"~\cite{sutton2018reinforcement} challenges in reinforcement learning.

Therefore, devising an excellent sequence modeling algorithm is essential for the new generation of offline RL. 
The diffusion probability model~\cite{rombach2022high,ramesh2022hierarchical}, with its demonstrated success in generative sequence modeling for natural language processing and computer vision, presents an ideal fit for this endeavor.
It also shows great potential as a paradigm for planning and decision-making.
For example, diffusion-based planning methods~\cite{janner2022planning, ajay2022conditional, wang2022diffusion} train trajectory diffusion models based on offline data and apply flexible constraints on generated trajectories through reward guidance during sampling.
In consequence, diffusion planners show notable performance superiority compared with transformer-based planners like Decision Transformer~\cite{chen2021decision} and Trajectory Transformer~\cite{janner2021offline} on long horizon tasks, while enabling goal-conditioned rather than reward-maximizing control at the same time.

While diffusion-based planners have achieved success in certain areas, their performance is limited by the lack of diversity in their training data.
In decision-making tasks, the cost of collecting a diverse set of offline training data may be high, and this insufficient diversity would impede the ability of the diffusion model to accurately capture the dynamics of the environment and the behavior policy.
As a result, diffusion models tend to perform inferior when expert data is insufficient, and particularly when facing new tasks.
This raises a natural question: can we use the generated heterogeneous data by the reward-guided diffusion model to improve the diffusion model itself since it has powerful generative sequence modeling capability?
As diffusion-based planners can generate quite diverse ``dream" trajectories for multiple tasks which may be different from the original task the training data are sampled from, greatly superior to Decision Transformer \cite{chen2021decision}, enabling the diffusion model to be self-evolutionary makes it a stronger planner, potentially benefiting more decision-making requirements and downstream tasks.
%


In this paper, we present \alias, a diffusion-based planner for goal-conditioned tasks that can generalize to novel settings and scenarios through self-evolution (see Figure~\ref{fig:teaser}). 
Unlike conventional approaches that rely heavily on specific expert data, \alias uses gradient of reinforcement learning rewards, directly integrated into the sampling process, as guidance to generate diverse and heterogeneous synthetic demonstration data for both existing and unseen tasks. 
The generated demonstration data is then filtered by a discriminator, of which the high-quality ones are used to fine-tune the diffusion model,
resulting in a better planner with significantly improved self-bootstrapping capabilities on previously seen tasks and an enhanced ability of generalizing to new tasks.
%
As a consequence, \alias not only improves the performance of the diffusion-based planner on existing benchmarks, but also enables it to adapt to unseen tasks without the need for additional expert data.

It's non-trivial to construct and evaluate \alias for both seen and unseen tasks.
We first conduct empirical experiments on two widely-used benchmarks (MuJoCo~\cite{todorov2012mujoco} and Maze2d) of D4RL~\cite{fu2020d4rl} to verify the self-bootstrapping capability of \alias on seen tasks.
Additionally, we creatively design new pick-and-place tasks based on previous stacking tasks in the KUKA~\cite{schreiber2010fast} industrial robot arm environment, and introduce novel auxiliary tasks (e.g., collecting gold coins) in Maze2D. 
The newly proposed tasks and settings provide an effective evaluation of the generalization capabilities of \alias on unseen tasks.


Our contributions are three-fold:
\textbf{1)} We present \alias, allowing diffusion-based planners to self-evolve for offline RL by generating high-quality heterogeneous data with reward-integrated diffusion model directly and filtering out inappropriate examples with a discriminator. 
\textbf{2)} We apply our self-evolutionary \alias to unseen (zero-shot) tasks without any additional expert data, demonstrating its strong generalization ability and adaptability.
\textbf{3)} Extensive experiments on two widely-used offline RL benchmarks from D4RL as well as our carefully designed unseen tasks in KUKA and Maze2d environments validate the effectiveness of \alias.







\section{Related Works}

\textbf{Offline Reinforcement Learning.}
Offline RL~\cite{levine2020offline,prudencio2022survey} is a popular research field that aims to learn behaviors using only offline data such as those collected from previous experiments or human demonstrations, without the need to interact with the live environment from time to time at the training stage. 
%

However, in practice, offline RL faces a major challenge that standard off-policy RL methods may fail due to the overestimation of values, caused by the distribution deviation between the offline dataset and the policy to learn. Most conventional offline RL methods use action-space constraints or value pessimism~\cite{buckman2020importance} to overcome the challenge~\cite{agarwal2020optimistic, kumar2020conservative, siegel2020keep, wu2019behavior, yang2022regularized}.  
For example, conservative Q-learning (CQL)~\cite{kumar2020conservative} addresses these limitations by learning a conservative Q-function, ensuring the expected value under this Q-function is lower than its true value.

\textbf{Reinforcement Learning as Sequence Modeling.}
Recently, a new paradigm for Reinforcement Learning (RL) has emerged, in which RL is viewed as a generic sequence modeling problem. 
It utilizes transformer-style models to model trajectories of states, actions, rewards and values, and turns its prediction capability into a policy that leads to high rewards. 
As a representative, Decision Transformer (DT)~\cite{chen2021decision} leverages a causally masked transformer to predict the optimal action, which is conditional on an autoregressive model that takes the past state, action, and expected return (reward) into account. It allows the model to consider the long-term consequences of its actions when making a decision. And based on DT, Trajectory Transformer (TT)~\cite{janner2021offline} is proposed to utilize transformer architecture to model distributions over trajectories, repurposes beam search as a planning algorithm, and shows great flexibility across long-horizon dynamics prediction, imitation learning, goal-conditioned RL, and offline RL. Bootstrapped Transformer~\cite{wang2022bootstrapped} further incorporates the idea of bootstrapping into DT and uses the learned model to self-generate more offline data to further improve sequence model training. However, Bootstrapped Transformer could not integrate RL reward into the data synthesizing process directly and can only amplify homogeneous data trivially for its original task, which can boost the performance but cannot enhance the adaptability on another unseen task. Besides, such approaches lack flexibility in adapting to new reward functions and tasks in different environments, as the generated data is not suitable for use in new tasks or environments. 

Diffuser~\cite{janner2022planning} presents a powerful framework for trajectory generation using the diffusion probabilistic model, which allows the application of flexible constraints on generated trajectories through reward guidance during sampling.
The consequent work, Decision Diffuser~\cite{ajay2022conditional} introduces conditional diffusion with reward or constraint guidance for decision-making tasks, further enhancing Diffuser's performance. 
Additionally, Diffusion-QL~\citep{wang2022diffusion}, adds a regularization term to the training loss of the conditional diffusion model, guiding the model to learn optimal actions. 
Nevertheless, the performance of these methods is still limited by the quality of offline expert data, leaving room for improvement in adapting to new tasks or settings.
\textbf{Diffusion Probabilistic Model.}
Diffusion models are a type of generative model that represents the process of generating data as an iterative denoising procedure \cite{sohl2015deep, ho2020denoising}. They have made breakthroughs in multiple tasks such as image generation \cite{song2020denoising}, waveform generation \cite{chen2020wavegrad}, 3D shape generation \cite{zhou2021shape} and text generation \cite{austin2021structured}. These models, which learn the latent structure of the dataset by modeling the way in which data points diffuse through the latent space, are closely related to score matching \cite{hyvarinen2005score} and energy-based models (EBMs) \cite{lecun06atutorial, du2019implicit, nijkamp2019learning, grathwohl2020stein}, as the denoising process can be seen as a form of parameterizing the gradients of the data distribution \cite{song2019generative}. 

Moreover, in the sampling process, diffusion models allow flexible conditioning \cite{dhariwal2021diffusion} and have the ability to generate compositional behaviors \cite{du2020compositional}. It shows that diffusion models own promising potential to generate effective behaviors from diverse datasets and plan under different reward functions including those not encountered during training.



\section{Preliminary}
Reinforcement Learning is generally modeled as a Markov Decision Process (MDP) with a fully observable state space, denoted as $\mathcal{M}=(\mathcal{S}, \mathcal{A}, \mathcal{T}, \mathcal{R}, \gamma)$, where $\mathcal{S}$ is the state space and $\mathcal{A}$ is the action space. 
Besides, $\mathcal{T}$ is the state transition function with the dynamics of this discrete-time system that $\bs_{t+1}=\mathcal{T}(\bs_t, \ba_t)$ at state $\bs_t \in \mathcal{S}$ given the action $\ba_t \in \mathcal{A}$. $\mathcal{R}(\bs_t, \ba_t)$ defines the reward function and $\gamma\in(0,1]$ is the discount factor for future reward.

Considering the offline reinforcement learning as a sequence modeling task, the objective of trajectory optimization is to find the optimal sequence of actions $\ba_{0:T}^*$ that maximizes the expected return with planning horizon $T$, which is the sum of per time-step rewards or costs $R(\bs_t, \ba_t)$:
\begin{equation}
\ba_{0:T}^*=\underset{\ba_{0:T}}{\arg \max } \mathcal{J}(\bs_0, \ba_{0:T})=\underset{\ba_{0:T}}{\arg \max } \sum_{t=0}^T \gamma^t R(\bs_t, \ba_t).
\end{equation}

The sequence data generation methods utilizing diffusion probabilistic models \cite{sohl2015deep, ho2020denoising} pose the generation process as an iterative denoising procedure, denoted by $p_\theta(\btau^{i-1} \mid \btau^i)$ where $\btau$ represents a sequence and $i$ is an indicator of the diffusion timestep.

Then the distribution of sequence data is expanded with the step-wise conditional probabilities of the denoising process,
\vspace{-2pt}
\begin{equation}
p_\theta\left(\btau^0\right)=\int p\left(\btau^N\right) \prod_{i=1}^N p_\theta\left(\btau^{i-1} \mid \btau^i\right) \mathrm{d} \btau^{1: N}
\vspace{-2pt}
\end{equation}
where $p\left(\btau^N\right)$ is a standard normal distribution and $\btau^{0}$ denotes original (noiseless) sequence data. 

The parameters $\theta$ of the diffusion model are optimized by minimizing the evidence lower bound (ELBO) of negative log-likelihood of $p_\theta\left(\btau^0\right)$, similar to the techniques used in variational Bayesian methods.
\begin{equation}
\theta^*=\arg \min _\theta-\mathbb{E}_{\btau^0}\left[\log p_\theta\left(\btau^0\right)\right]
\label{eq:theta_optim}
\end{equation}
What's more, as the denoising process is the reverse of a forward diffusion process which corrupts input data by gradually adding noise and is typically denoted by $q\left(\btau^i \mid \btau^{i-1}\right)$, the reverse process can be parameterized as Gaussian under the condition that the forward process obeys the normal distribution and the variance is small enough \cite{feller2015theory}.
\begin{equation}
p_\theta\left(\boldsymbol{\tau}^{i-1} \mid \boldsymbol{\tau}^i\right)=\mathcal{N}\left(\boldsymbol{\tau}^{i-1} \mid \mu_\theta\left(\boldsymbol{\tau}^i, i\right), \Sigma^i\right)
\label{eq:p1}
\end{equation}
in which $\mu_{\theta}$ and $\Sigma$ are the mean and covariance of the Gaussian distribution respectively.

For model training, with the basis on Eq.~\ref{eq:theta_optim} and \ref{eq:p1}, \cite{ho2020denoising} proposes a simplified surrogate loss:
\begin{equation}
\label{eq:diff_loss}
    \mathcal{L}_{\text{denoise}}(\theta) \coloneqq \mathbb{E}_{i, \btau^{0} \sim q, \epsilon \sim \mathcal{N}}[||\epsilon - \epsilon_{\theta}(\btau^i, i)||^{2}]
\end{equation}
where $i \in \{0, 1, ..., N\}$ is the diffusion timestep, $\epsilon \sim \mathcal{N}(\boldsymbol{0}, \boldsymbol{I})$ is the target noise, and $\btau^{i}$ is the trajectory $\btau^{0}$ corrupted by noise $\epsilon$ for $i$ times. This is equivalent to predicting the mean $\mu_{\theta}$ of $p_\theta\left(\boldsymbol{\tau}^{i-1} \mid \boldsymbol{\tau}^i\right)$ as the function mapping from $\epsilon_{\theta}(\btau^i, i)$ to $\mu_{\theta}(\btau^i, i)$ is a closed-form expression.

\section{Method}
In this section, we first introduce the basic planning with the diffusion method and its limitations. Then, we propose \alias, a novel self-evolved sequence modeling method for decision-making with the basis of diffusion probabilistic models.
\alias is designed to enhance the performance of diffusion models in existing decision-making tasks, especially the goal-conditioned tasks, and further improve their adaptability in unseen tasks without any expert data to supervise the training process.

\subsection{Planning with Task-oriented Diffusion Model}
Following previous work \cite{janner2022planning}, we can re-define the planning trajectory as a special kind of sequence data with actions as an additional dimension of states like:
\begin{equation}
\btau=\left[\begin{array}{llll}
\bs_0 & \bs_1 & ... & \bs_T \\
\ba_0 & \ba_1 & ... & \ba_T
\end{array}\right]
\label{eq:tau}
\end{equation}
Then we can use the diffusion probabilistic model to perform trajectory generation. However, the aim of planning is not to restore the original trajectory but to predict future actions with the highest reward-to-go, the offline reinforcement learning should be formulated as a conditional generative problem with guided diffusion models that have achieved great success on image synthesis \cite{dhariwal2021diffusion}. So, we drive the conditional diffusion process:
\begin{equation}
\label{eq:diff_plan}
    q(\btau^{i+1} | \btau^i), \;\;\;\; p_{\theta}(\btau^{i-1}|\btau^{i}, \boldsymbol{y}(\btau))
\end{equation}
where the new term $\boldsymbol{y}(\btau)$ is some specific information of the given trajectory $\btau$, such as the reward-to-go (return) $\mathcal{J}(\btau^0)$ of the trajectory, the constraints that must be satisfied by the trajectory and so on. On this basis, we can rewrite the optimization objective as,
\begin{equation}
 \theta^*=\arg \min _\theta-\mathbb{E}_{\btau^0}\left[\log p_\theta(\btau^{0} | \boldsymbol{y}(\btau^0))\right]
\label{eq:cond_gen_model}
\end{equation}

\begin{figure}[t]
\centering 
\includegraphics[width=0.95\linewidth]{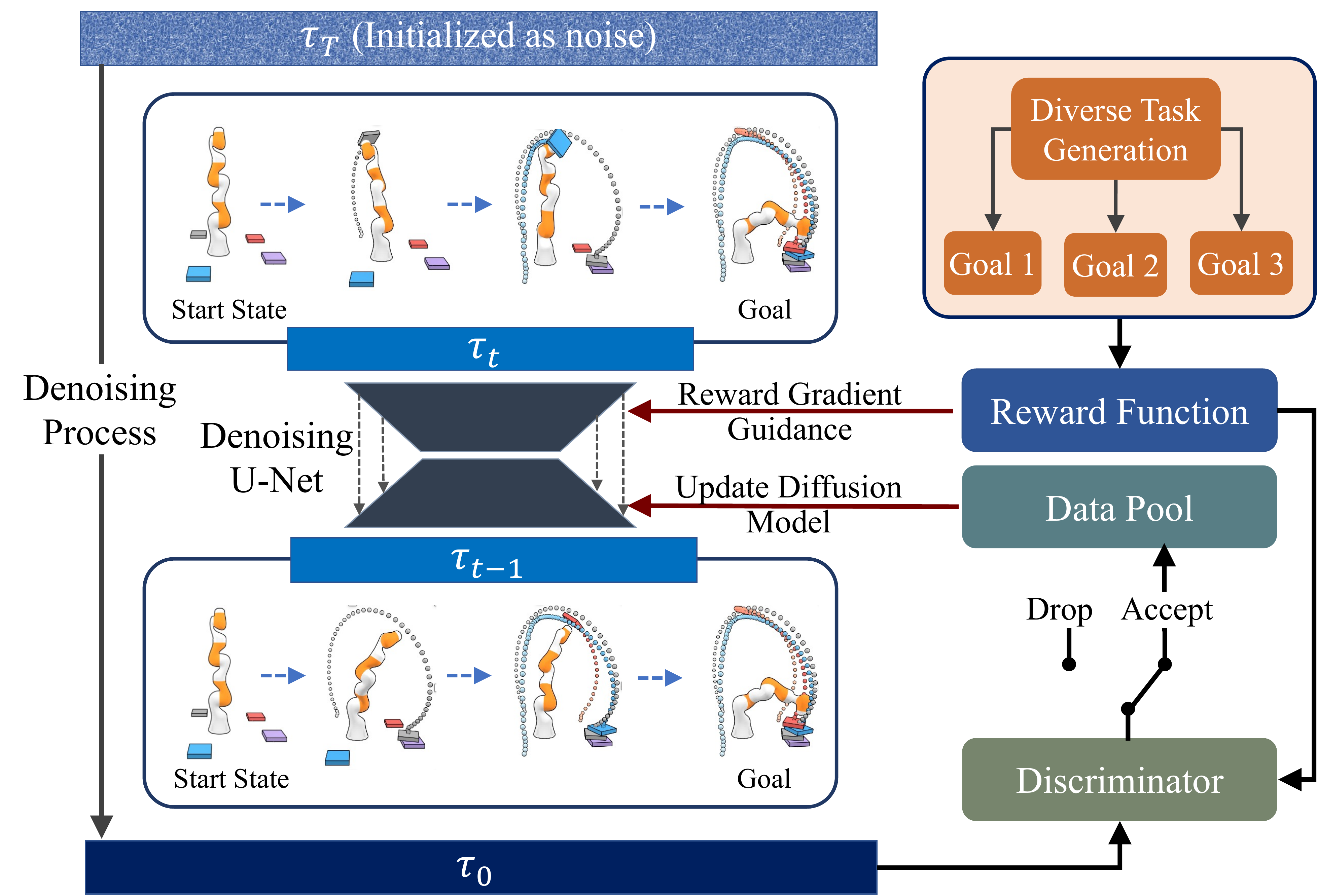}
\caption{Overall framework of \alias. 
To improve the adaptability of the diffusion model to diverse tasks, rich data with distinct objectives is generated, guided by each task’s reward function. 
%
During the diffusion denoising process, we utilize a pre-trained denoising U-Net to progressively generate high-quality trajectories.
At each denoising time step, we take the
task-specific reward of a trajectory to adjust the gradient of state and action sequence, thereby creating trajectories that align with specific task objectives. 
Subsequently, the generated synthetic trajectory is evaluated by a discriminator to see if it meets the standards. If yes, it is incorporated into a data pool to fine-tune the diffusion model. 
The procedure iteratively enhances the generalizability of our model for both seen and unseen settings.}
\label{fig:main_frame} 
\end{figure}

Therefore, for tasks aiming to maximize the reward-to-go, we take $\mathcal{O}_{t}$ to denote the optimality of the trajectory at timestep $t$. And $\mathcal{O}_{t}$ obeys Bernoulli distribution with ${p(\mathcal{O}_t=1) = \exp (\gamma^{t} \mathcal{R}(\bs_t, \ba_t))}$. When $p(\mathcal{O}_{1:T} \mid \btau^{i})$ meets specific Lipschitz conditions, the conditional transition probability of the reverse diffusion process can be approximated as \cite{feller2015theory}:
\begin{equation}
\label{eq:guided}
p_\theta(\btau^{i-1} \mid \btau^{i}, \mathcal{O}_{1:T}) \approx \mathcal{N}(\btau^{i-1}; \mu_{\theta} + \alpha\Sigma g, \Sigma)
\end{equation}

\vspace{-20pt}
\begin{align*}
\text{where, } g &= \nabla_{\btau} \log p(\mathcal{O}_{1:T} \mid \btau) |_{\btau = \mu_{\theta}} \\
&= \sum_{t=0}^{T} \gamma^{t} \nabla_{\bs_t,\ba_t} \mathcal{R}(\bs_t, \ba_t) |_{(\bs_t,\ba_t)=\mu_t}
= \nabla_{\btau} \mathcal{J}(\mu_{\theta}).
\end{align*}
\vspace{-.45cm}

Besides, for tasks aiming to satisfy single point conditional constraint (e.g. goal conditioned tasks), the constraint can be simplified by substituting conditional values for the sampled values of all diffusion timesteps $i \in \{0, 1, ..., N\}$.

Although this paradigm has achieved competitive results with previous planning methods which are not based on diffusion models, it only performs conditional guidance during the reverse diffusion process and assumes the unconditional diffusion model is trained perfectly over the forward process. 
However, as depicted in Eq.~\ref{eq:guided}, the quality of generated trajectory $\btau$ depends not only on the guided gradient $g$ but more on the learned means $\mu_{\theta}$ and covariance $\Sigma$ of the unconditional diffusion model. If the learned $\mu_{\theta}$ deviates far from the optimal trajectory, no matter how strong the guidance $g$ is, the final generated result will be highly biased and of low quality. 
Then, learning from Eq.~\ref{eq:diff_loss}, the quality of $\mu_{\theta}$ hinges on the training data, the quality of which, however, is uneven across different tasks, especially on unseen tasks. Previous diffusion-based planning methods have not solved the problem which limits the performance of these methods on both existing and unseen tasks, and thus have poor adaptation ability.

\subsection{Self-evolved Planning with Diffusion}
Therefore, with the aim to improve the adaptability of these planners, we propose \alias, a novel self-evolved decision-making approach based on diffusion probabilistic models, to enhance the quality of the trained means $\mu_{\theta}$ and covariance $\Sigma$ of the forward diffusion process.
\alias relies on self-evolved synthetic data generation to enrich the training dataset which is denoted as $\btau_0$ and synthetic data fine-tuning to boost performance. After that, \alias follows the paradigm depicted in Eq.~\ref{eq:guided} to find the optimal action sequence for the given task with the guidance of reward gradients.

As shown in Figure \ref{fig:main_frame}, to implement \alias, we firstly generate a large number of synthetic demonstration data for unseen tasks which do not exist in the training dataset in order to simulate a wide range of scenarios and behaviors that the diffusion model may encounter in the real world.
This synthetic data is iteratively generated through the sampling process of the original diffusion probabilistic model $\theta_0^*$ with reward guidance, taking the advantage of its great generation ability. We will discuss the details of the synthetic data generation in Section \ref{sec:method_gen} and here we just abbreviate it as a function $\mathcal{G}(\mu_{\theta}, \Sigma, \nabla_{\btau} \mathcal{J}(\mu_{\theta}))$.

Secondly, we design a rule-based discriminator $\mathcal{D}$, with reward and dynamics consistency guidance, to select high-quality data from the generated data pool. Previous sequence modeling methods which predict the rewards $\mathcal{R}(\bs, \ba)$ simultaneously with generated states and actions are unable to solve the dynamics consistency problem that the actual next state with transition model $\bs' = \mathcal{T}(\bs, \ba)$ greatly deviates from the predicted next state. What's more, these deviated trajectories are taken as feasible solutions under previous settings.

To resolve this problem, \alias only takes the state sequence $\bs = \left[\bs_0, \bs_1, ..., \bs_T \right]$ of the generated trajectory and then performs state tracking control using a traditional or neural network-based inverse dynamics model $\mathcal{I}$ to derive real executable actions, denoted as $\widetilde{\ba}_t = \mathcal{I}(\bs_t, \bs_{t+1})$. This step ensures the action that does not violate the robot's dynamic constraints. After that, \alias performs $\widetilde{\ba}_t$ to obtain the revised next state $\widetilde{\bs}_{t+1}=\mathcal{T}\left(\widetilde{\bs}_t, \widetilde{\ba}_t\right)$, and then filters out the trajectories whose revised state $\widetilde{\bs}_{t+1}$ has a too large difference from the generated $\bs_{t+1}$ (measured by MSE $d=||\widetilde{\bs}_{t+1} - \bs_{t+1}||_2$). The remaining trajectories $\widetilde{\bs}$ are then used to predict the reward by $\widetilde{\mathcal{R}}=\mathcal{R}(\widetilde{\bs}, \widetilde{\ba})$ with the new actions $\widetilde{\ba}$ and are selected according to this reward. In this way, we can derive high-quality synthetic data to fine-tune the diffusion probabilistic model. 

%

We repeat this process multiple times in order to continually improve the model's performance and adapt it to new tasks, ultimately improving its generalization performance. So, it can be formulated as,
\begin{equation}
\begin{split}
&\theta_k^* \quad =\arg \min _\theta-\mathbb{E}_{\hat{\btau}_k}\left[\log p_\theta(\hat{\btau}_k | \boldsymbol{y}(\hat{\btau}_k))\right] \\
 &\btau_{k+1} = \mathcal{G}\left(\mu_{\theta_k^*}, \Sigma, \nabla_{\btau} \mathcal{J}(\mu_{\theta_k^*})\right) \\ 
 &\hat{\btau}_{k+1} = [\hat{\btau_k}, \mathcal{D}(\widetilde{\mathcal{R}}(\btau_{k+1})) ]
\end{split}
\end{equation}
where $k \in \{0, 1, ...\}$ is the number of iteration rounds and the initial dataset $\hat{\btau}_0 = \btau_0$.

\subsection{Reward-guided Synthetic Data Generation}
\label{sec:method_gen}
To improve the performance and adaptability of the diffusion probabilistic model on unseen tasks, we need to generate synthetic trajectory data using the learned diffusion model at the current iteration. We achieve it by defining a series of tasks with different goals and reward functions. 

\textbf{Continuous Reward Function.} For the tasks with continuous reward function, represented by MuJoCo \cite{todorov2012mujoco}, we follow the settings that define a binary random variable indicating the optimality with probability mapped from a continuous value, to convert the reward maximization problem to a continuous optimization problem. We can easily take Eq.~\ref{eq:guided} to generate synthetic results.

\textbf{Sparse Reward Function.} The reward function of tasks as typified by a goal-conditioned problem like Maze2D is a unit step function $\mathcal{J}(\btau) = \boldsymbol{\chi}_{\bs_g }(\btau)$ whose value is equal to 1 if and only if the generated trajectory contains the goal state $\bs_g$. 
The gradient of this reward function is Dirac delta function ~\cite{zhang2021dirac} which is not a classical function and cannot be adopted as guidance. However, if it is considered from the perspective of taking the limit, the constraint can be simplified as replacing all corresponding sampled values with constraints over the diffusion timesteps.

\textbf{Combination.} Many realistic tasks need these two sorts of reward functions simultaneously. For example, if there exists an auxiliary task in Maze2D environment that requires the planner to not only find a way from the start point to the goal point but also collect the gold coin in the maze. This task is more difficult and it's infeasible to add this constraint to the sparse reward term because there is no idea about which timestep the generated trajectory should pass the additional reward point (denoted as $\bs_c$). As a solution, we propose to combine these two sorts of methods and define an auxiliary reward guiding function to satisfy the constraints.
\vspace{-4pt}
\begin{equation}
    \mathcal{J}(\btau) = \sum_{t=0}^{T} ||\bs_t - \bs_c||_p
\label{eq:newj}
\vspace{-4pt}
\end{equation}
where $p$ represents p-norm. Then, with Eq.~\ref{eq:newj} we plug it into Eq.~\ref{eq:guided} as the marginal probability density function and force the last state of the generated trajectory $\btau^0$ to be $\bs_c$.
The generated trajectories that meet the desired criteria of the discriminator are added to the set of training data for the diffusion model learning as synthetic expert data. This process is repeated multiple times until a sufficient amount of synthetic data has been generated. By iteratively generating and selecting high-quality data based on the guidance of expected return and dynamics transition constraints, we can boost the performance and enhance the adaptability of the diffusion probabilistic model.

\section{Experiment}
\subsection{Benchmarks}
\textbf{Maze2D}: Maze2D \cite{fu2020d4rl} environment is a navigation task in which a 2D agent needs to traverse from a randomly designated location to a fixed goal location where a reward of 1 is given. No reward shaping is provided at any other location. The objective of this task is to evaluate the ability of offline RL algorithms to combine previously collected sub-trajectories in order to find the shortest path to the evaluation goal. Three maze layouts are available: ``umaze", ``medium", and ``large". The expert data for this task is generated by selecting random goal locations and using a planner to generate sequences of waypoints that are followed by using a PD controller to perform dynamic tracking. We also provide a method to derive more diverse layouts with ChatGPT in Appendix \ref{appendix:chatgpt}.

\textbf{MuJoCo}: 
MuJoCo \cite{todorov2012mujoco} is a physics engine that allows for real-time simulation of complex mechanical systems. It has  three typical tasks: Hopper, HalfCheetah, and Walker2d. Each task has 4 types of datasets to test the performance of an algorithm: ``medium", ``random", ``medium-replay" and ``medium-expert". The ``medium" dataset is created by training a policy with a certain algorithm and collecting $1$M samples. The ``random" dataset is created by using a randomly initialized policy. The ``medium-replay" dataset includes all samples recorded during training until the policy reaches a certain level of performance. There is also a ``medium-expert" dataset which is a mix of expert demonstrations and sub-optimal data.


\textbf{KUKA Robot}: The KUKA Robot \cite{schreiber2010fast} benchmark is a standardized evaluation tool that is self-designed to measure the capabilities of a robot arm equipped with a suction cup at the end. It consists of two tasks: conditional stacking \cite{janner2022planning} and pick-and-place. More details can be seen in Sec. \ref{sec:pick_place}. By successfully completing these tasks, the KUKA Robot benchmark can accurately assess the performance of the robot arm and assist developers in improving its design.

\subsection{Performance Enhancement on Existing Tasks}
\subsubsection{Experiments on Maze2d Environment}
\vspace{-1pt}
\textbf{Overall Performance.}
Navigation in Maze2D environment takes planners hundreds of steps to reach the goal location. Even the best model-free algorithms have to make great efforts to adequately perform credit assignments and reliably reach the target. 
We plan with \alias using the strategy of sparse reward function to condition on the start and goal location. We compare our method with the best model-free algorithms (IQL \citealt{kostrikov2021offline} and CQL \citealt{kumar2020conservative}), conventional trajectory
optimizer MPPI \cite{williams2015model} and previous diffusion-based approach Diffuser \cite{janner2022planning} in Table \ref{table:maze2d}. This comparison is fair because model-free methods can also identify the location of the goal point which is the only state with a non-zero reward.

\begin{table}[t]
\caption{
    \textbf{Offline Reinforcement Learning Performance in Maze2d Environment.} We show the results of \alias and previous planning methods to validate the bootstrapping effect of our method on a goal-conditioned task.
}
\label{table:maze2d}
\vspace{0.2cm}
\centering
\small
\tabcolsep 3pt
\begin{tabular}{cccccc}
\toprule
\multicolumn{1}{c}{\textbf{Environment}} & \textbf{MPPI} & \textbf{CQL} & \textbf{IQL} & \textbf{Diffuser} & \textbf{AdaptDiffuser} \\
\midrule
U-Maze & 33.2 & 5.7 & 47.4 & 113.9 & \textbf{135.1} \scriptsize{\raisebox{1pt}{$\pm 5.8$}}\\
Medium & 10.2 & 5.0 & 34.9 & 121.5 & \textbf{129.9} \scriptsize{\raisebox{1pt}{$\pm 4.6$}}\\
Large  & 5.1 & 12.5 & 58.6 & 123.0 & \textbf{167.9} \scriptsize{\raisebox{1pt}{$\pm 5.0$}}\\
\midrule
\multicolumn{1}{c}{\textbf{Average}} & 16.2 & 7.7 & 47.0 & 119.5 & \textbf{144.3} \hspace{.58cm} \\
\bottomrule
\end{tabular}
\end{table}

\begin{figure}[t]
    \begin{center}
    \centerline{\includegraphics[width=1.02\columnwidth]{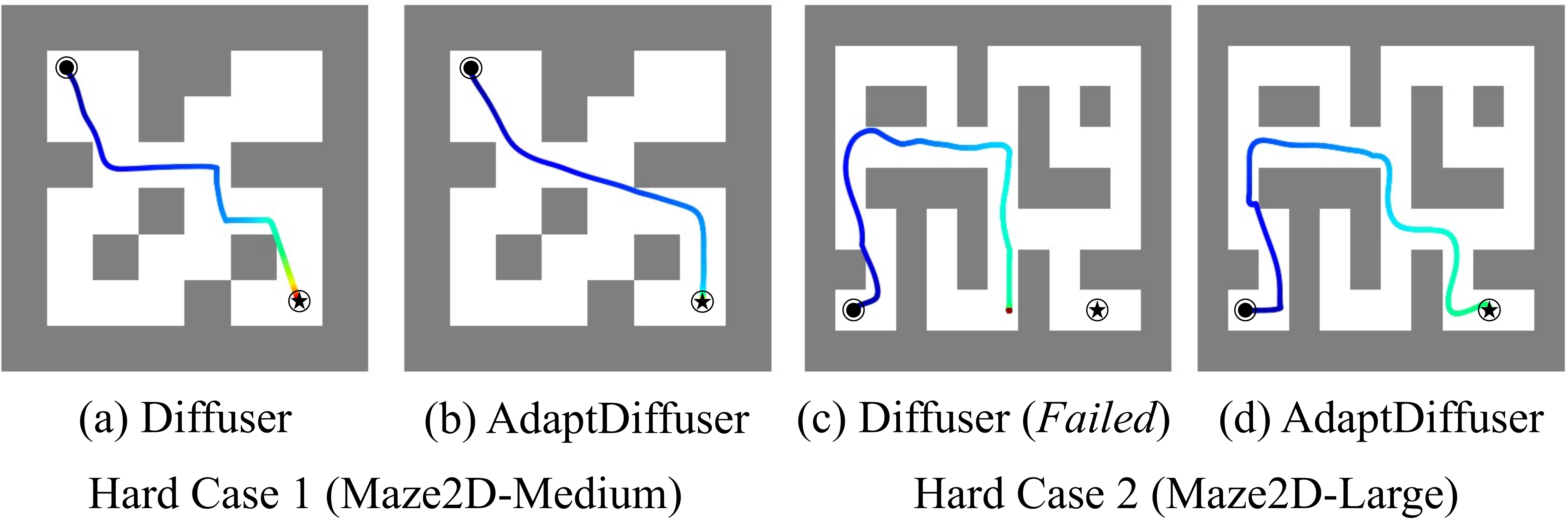}}
    \vspace{-0.05in}
    \caption{\textbf{Hard Cases of Maze2D with Long Planning Path.} Paths are generated in the Maze2D environment with a specified start 
    \protect{\raisebox{-.05cm}{\includegraphics[height=.3cm]{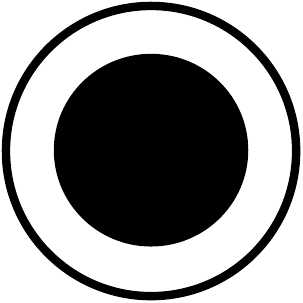}}}
    and goal
    \protect{\raisebox{-.05cm}{\includegraphics[height=.3cm]{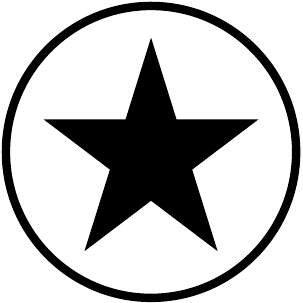}}}
    condition.}
    \label{fig:maze-hard}
    \end{center}
    \vspace{-20pt}
\end{figure}

\begin{table*}[t]
\caption{\small \textbf{Offline Reinforcement Learning Performance in MuJoCo Environment.} 
We report normalized average returns of D4RL tasks \citep{fu2020d4rl} in the table. And the mean and the standard error are calculated over 3 random seeds.
}
\label{table:mujoco}
\vspace{0.2cm}
\centering
\small
\tabcolsep 4.5pt
\begin{tabular}{ccccccccccccc}
\toprule
\textbf{Dataset} & \textbf{Environment} & \textbf{BC} & \textbf{CQL} & \textbf{IQL} & \textbf{DT} & \textbf{TT} & \textbf{MOPO} & \textbf{MOReL} & \textbf{MBOP} & \textbf{Diffuser} &  \textbf{AdaptDiffuser}\\ 
\midrule
Med-Expert & HalfCheetah & $55.2$ & $91.6$ & $86.7$ & $86.8$ & $95.0$ & $63.3$ & $53.3$ & $\textbf{105.9}$ & $88.9$ & $89.6$ \scriptsize{\raisebox{1pt}{$\pm 0.8$}} \\ 
Med-Expert & Hopper & $52.5$ & $105.4$ & $91.5$ & $107.6$ & $\textbf{110.0}$ & $23.7$ & $108.7$ & $55.1$ & $103.3$ & $\textbf{111.6}$ \scriptsize{\raisebox{1pt}{$\pm 2.0$}} \\ 
Med-Expert & Walker2d & $\textbf{107.5}$ & $\textbf{108.8}$ & $\textbf{109.6}$ & $\textbf{108.1}$ & $101.9$ & $44.6$ & $95.6$ & $70.2$ & $\textbf{106.9}$ & 
$\textbf{108.2}$ \scriptsize{\raisebox{1pt}{$\pm 0.8$}}
\\ 
\midrule
Medium & HalfCheetah & $42.6$ & $44.0$ & $\textbf{47.4}$ & $42.6$ & $\textbf{46.9}$ & $42.3$ & $42.1$ & $44.6$ & $42.8$ & $44.2$ \scriptsize{\raisebox{1pt}{$\pm 0.6$}} \\ 
Medium & Hopper & $52.9$ & $58.5$ & $66.3$ & $67.6$ & $61.1$ & $28.0$ & $\textbf{95.4}$ & $48.8$ & $74.3$ & $\textbf{96.6}$ \scriptsize{\raisebox{1pt}{$\pm 2.7$}}
 \\ 
Medium & Walker2d & $75.3$ & $72.5$ & $78.3$ & $74.0$ & $79.0$ & $17.8$ & $77.8$ & $41.0$ & $79.6$ & $\textbf{84.4}$ \scriptsize{\raisebox{1pt}{$\pm 2.6$}} \\ 
\midrule
Med-Replay & HalfCheetah & $36.6$ & $\textbf{45.5}$ & $\textbf{44.2}$ & $36.6$ & $41.9$ & $53.1$ & $40.2$ & $42.3$ & $37.7$ & $38.3$ \scriptsize{\raisebox{1pt}{$\pm 0.9$}}
\\ 
Med-Replay & Hopper & $18.1$ & $95.0$ & $94.7$ & $82.7$ & $91.5$ & $67.5$ & $\textbf{93.6}$ & $12.4$ & $\textbf{93.6}$ &
$\textbf{92.2}$ \scriptsize{\raisebox{1pt}{$\pm 1.5$}} \\ 
Med-Replay & Walker2d & $26.0$ & $77.2$ & $73.9$ & $66.6$ & $82.6$ & $39.0$ & $49.8$ & $9.7$ & $70.6$ & $\textbf{84.7}$ \scriptsize{\raisebox{1pt}{$\pm 3.1$}} \\ 
\midrule
\multicolumn{2}{c}{\textbf{Average}} & 51.9 & 77.6 & 77.0 & 74.7 & 78.9 & 42.1 & 72.9 & 47.8 & 77.5 & \textbf{83.4} \hspace{.58cm} \\
\bottomrule
\end{tabular}
\vspace{-4pt}
\end{table*}

As shown in Table \ref{table:maze2d}, scores achieved by \alias are over 125 in all maze sizes and are 20 points higher than those of Diffuser in average, indicating our method's strong effectiveness in goal-conditioned tasks.

\textbf{Visualization of Hard Cases.}
In order to more intuitively reflect the improvement of our method compared with previous Diffuser \cite{janner2022planning}, we select one difficult planning example of Maze2D-Medium and one of Maze2D-Large respectively for visualization, as shown in Figure \ref{fig:maze-hard}. Among the Maze2D planning paths with sparse rewards, the example with the longest path to be planned is the hardest one. Therefore, in Maze2D-Medium (Fig. \ref{fig:maze-hard} (a) (b)), we designate the start point as (1, 1) with goal point (6, 6), while in Maze2D-Large (Fig. \ref{fig:maze-hard} (c) (d)), we specify the start point as (1, 7) with goal point (9, 7) in the figure.

It can be observed from Fig. \ref{fig:maze-hard} that in Hard Case 1, \alias generates a shorter and smoother path than that generated by Diffuser. So, AdaptDiffuser achieves a larger reward. And in Hard Case 2, previous Diffuser method even fails to plan while our \alias derives a feasible path.

\subsubsection{Experiments on MuJoCo Environment}
\label{sec:mujoco}
MuJoCo tasks are employed to test the performance enhancement of our \alias learned from heterogeneous data of varying quality using the publicly available D4RL datasets \citep{fu2020d4rl}. We evaluate our approach with a number of existing algorithms that cover a variety of data-driven methodologies, including model-free RL algorithms like CQL \cite{kumar2020conservative} and IQL \cite{kostrikov2021offline}; return-conditioning approaches like Decision Transformer (DT) \cite{chen2021decision}; and model-based RL algorithms like Trajectory Transformer (TT) \cite{janner2021offline}, MOPO \cite{yu2020mopo}, MOReL \cite{kidambi2020morel}, and MBOP \cite{argenson2020model}. The results are shown in Table \ref{table:mujoco}. Besides, it is also worth noting that in the MuJoCo environment, the state sequence $\widetilde{\bs}$ derived by taking the generated actions $\ba$ is very close to the generated state sequence $\bs$, so we directly use $\widetilde{\mathcal{R}}(\bs, \ba)=\mathcal{R}(\bs, \ba)$ in this dataset.

\begin{figure}[t]
    \begin{center}
\centerline{\includegraphics[width=0.99\columnwidth]{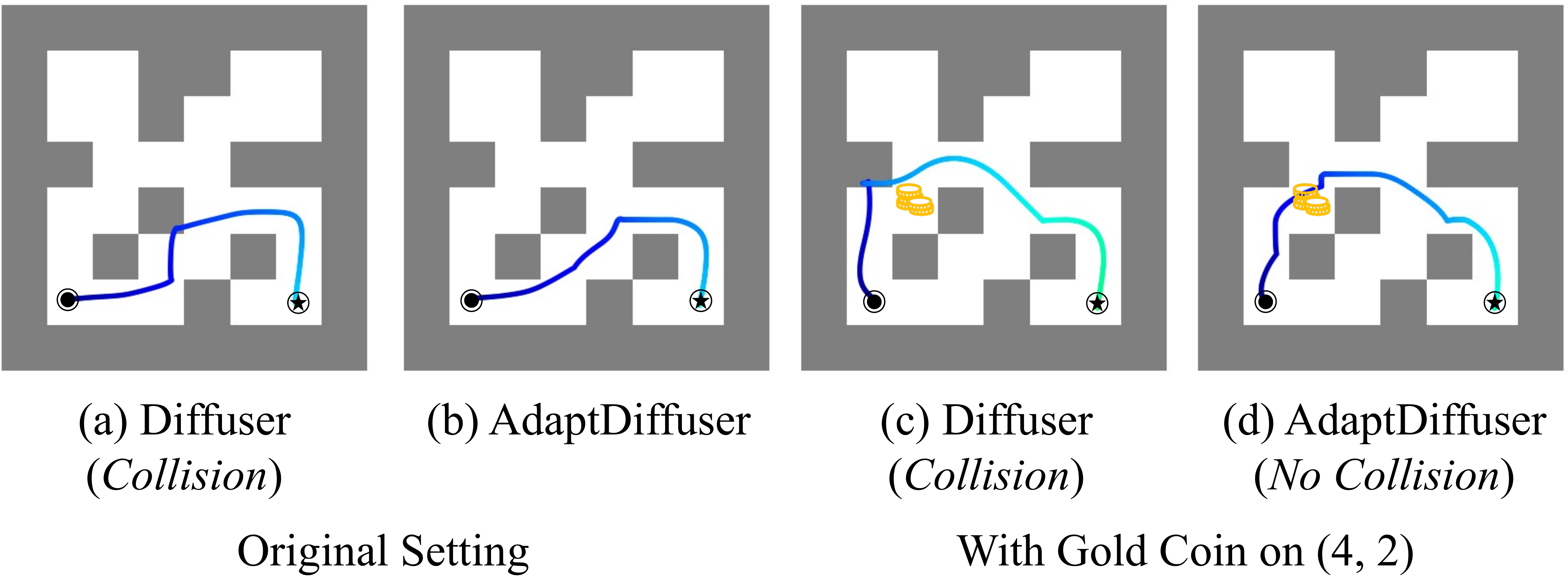}}
    \vspace{-0.05in}
    \caption{\textbf{Maze2d Navigation with Gold Coin Picking Task.} Subfigures (a) (b) show the optimal path when there are no gold coins in the Maze. (The generated routes walk at the bottom of the Maze.) 
    And subfigures (c) (d) add additional reward \protect{\raisebox{-.10cm}{\includegraphics[height=.4cm]{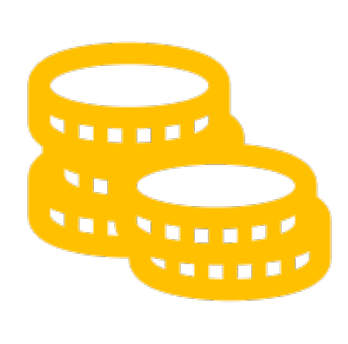}}} in (4,2) position of the Maze. The planners generate new paths that pass through the gold coin as shown in subfigures (c) (d). (The newly generated routes walk in the middle of the maze.)
    }
    \label{fig:maze_adapt}
    \end{center}
    \vspace{-0.3in}
\end{figure}

Observed from the table, our method \alias is either competitive or outperforms most of the offline RL baselines across all three different locomotion settings. And more importantly, compared with Diffuser \cite{janner2022planning}, our method achieves higher reward in almost all the datasets and improves the performance greatly, especially in ``Hopper-Medium" and ``Walker2d-Medium" environments. 
We analyze that this is because the quality of the original data in the ``Medium dataset" is poor, so \alias has an evident effect on improving the quality of the training dataset, thus significantly enhancing the performance of the planner based on the diffusion probabilistic model. 
The results of the ``Medium-Expert" dataset verify this analysis because the quality of original data in the ``Medium-Expert" dataset (especially the Halfcheetah environment) has been good enough, making the generation of new data only has a little gain on the model performance.

\subsection{Adaptation Ability on Unseen Tasks}
\subsubsection{Maze2d with Gold Coin Picking Task}
On top of existing Maze2D settings, we carefully design a new task that requires the agent to navigate as well as pick all gold coins in the maze. We show an example with an additional reward in (4, 2) in Figure \ref{fig:maze_adapt}.

We can see that when there is no additional reward, both Diffuser \cite{janner2022planning} and our method \alias choose the shorter path at the bottom of the figure to reach the goal point. But, when additional reward is added in the (4, 2) position of the maze, both planners change to the path walking in the middle of the figure under the guidance of rewards. However, at this time, the path generated by Diffuser causes the agent to collide with the wall, while \alias generates a smoother collision-free path, reflecting the superiority of our method.

\begin{table}[t]
\caption{
    \textbf{Adaptation Performance on Pick-and-Place Task}
}
\vspace{0.2cm}
\label{table:kukapick}
\tabcolsep 8pt
\centering
\small
\begin{tabular}{ccc}
\toprule
\multicolumn{1}{c}{\textbf{Environment}} & \textbf{Diffuser} & \textbf{AdaptDiffuser} \\
\midrule
Pick and Place setup 1 & 28.16 \scriptsize{\raisebox{1pt}{$\pm 2.0$}} & \textbf{36.03} \scriptsize{\raisebox{1pt}{$\pm 2.1$}} \\
Pick and Place setup 2 & 35.25 \scriptsize{\raisebox{1pt}{$\pm 1.4$}} & \textbf{39.00} \scriptsize{\raisebox{1pt}{$\pm 1.3$}} \\
\midrule
\multicolumn{1}{c}{\textbf{Average}} & 31.71 \hspace{.58cm} & \textbf{37.52} \hspace{.58cm} \\
\bottomrule
\end{tabular}
\vspace{-5pt}
\end{table}

\begin{figure}[t]
    \begin{center}
    \centerline{\includegraphics[width=0.9\columnwidth]{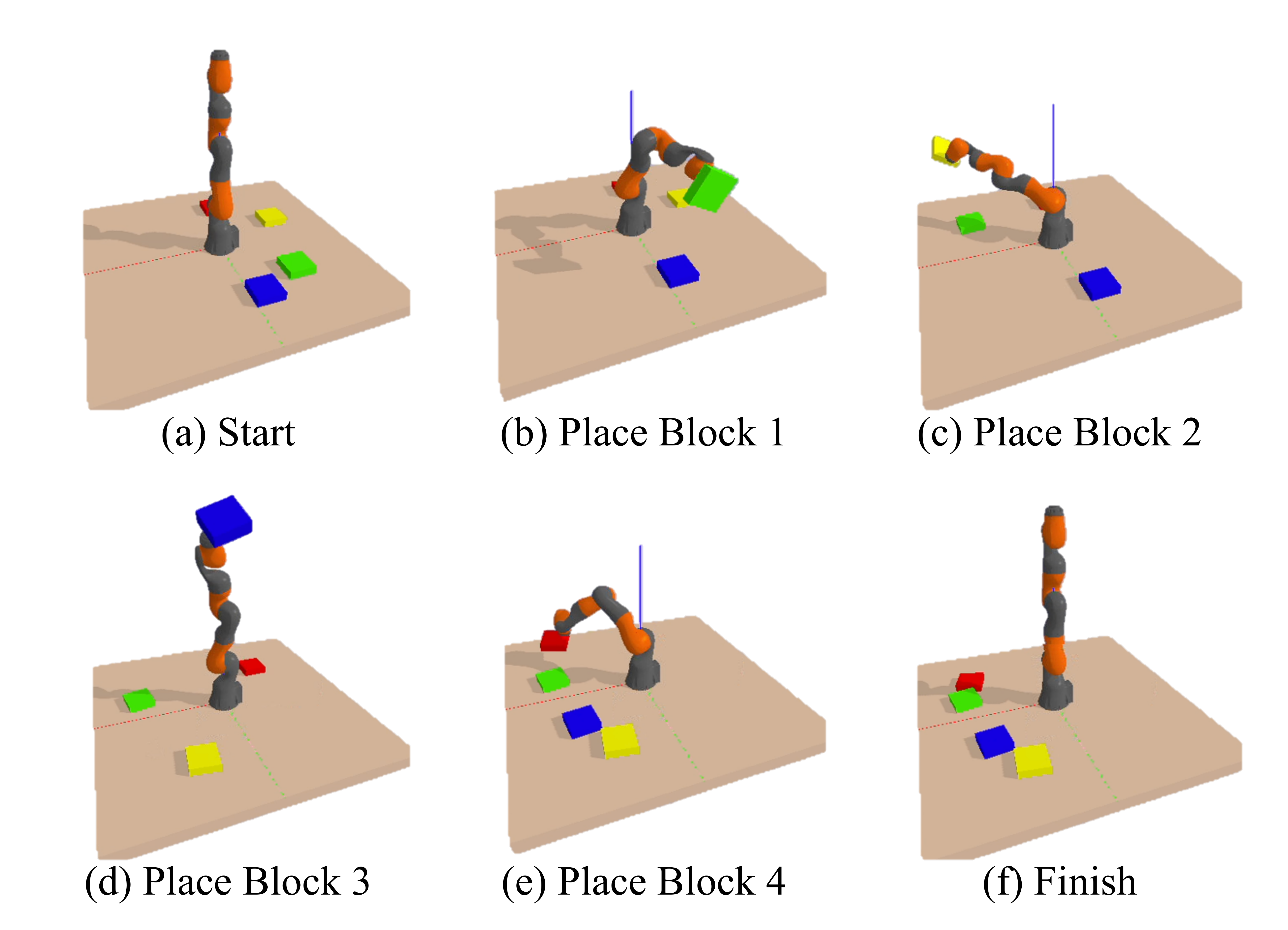}}
    \vspace{-0.05in}
    \caption{\textbf{Visualization of KUKA Pick-and-Place Task.} We require the KUKA Arm to move the blocks from their random initialized positions on the right side of the table to the left and arrange them in the order of yellow, blue, green, and red (from near to far).}
    \label{fig:pick_vis}
    \end{center}
\vspace{-15pt}
\end{figure}

\subsubsection{KUKA Pick and Place Task}
\label{sec:pick_place}
\textbf{Task Specification.}
There are two tasks in the KUKA robot arm environment. One is the conditional stacking task, as defined in \cite{janner2022planning}, where the robot must correctly stack blocks in a predetermined order on a designated location, using blocks that have been randomly placed. 
And the other is the pick-and-place task designed by us, which aims to place the randomly initialized blocks in their own target locations in a predetermined order. The reward functions of both tasks are defined as one upon successful placements and zero otherwise.

To test the adaptation capability of \alias and other baselines, we only provide expert trajectory data for the conditional stacking task, which is generated by PDDLStream \cite{garrett2020pddlstream}, but we require the planner to generalize to pick-and-put task without any expert data. The performance of the pick-and-place task is supposed to be a good measure of the planner's adaptability.



\textbf{Adaptation Performance.}
In KUKA pick-and-place task, we define the guidance of the conditional diffusion model as the gradient of the reward function about the distance between the current location and the target location. Then, the adaptation performance is displayed in Table \ref{table:kukapick}.

There are two setups in KUKA benchmark. In setup 1, the four blocks are initialized randomly on the floor, while in setup 2, the four blocks are stacked at a random location at the beginning.
As shown in Table \ref{table:kukapick}, \alias outperforms Diffuser greatly on both setups while achieving higher performance at setup 2 because all of the blocks start from the same horizontal position. We visualize a successful case of the KUKA pick-and-place task in Figure \ref{fig:pick_vis}, and more visualization results can be seen in Appendix \ref{appedix:vis_kuka}.

\subsection{Ablation Study}

\subsubsection{Ablation on Iterative Phases}
In order to verify the lifting effect of iterative data generation of our method \alias to improve the performance of the planner, we conduct an ablation experiment on the number of iterative phases of \alias in the MuJoCo environment of D4RL.
\begin{table}[t]
\caption{
\textbf{Ablation on Iterative Phases.} The mean and the standard error are calculated over 3 random seeds.
}
\label{table:mujoco_iterphases}
\vspace{0.2cm}
\centering
\footnotesize
\begin{tabular}{cccc}
\toprule
\textbf{Dataset} & \textbf{Environment} & \textbf{$1^{\text{st}}$ Phase} & \textbf{$2^{\text{nd}}$ Phase} \\ 
\midrule
Medium-Expert & HalfCheetah & $\textbf{89.3}$ \scriptsize{\raisebox{1pt}{$\pm 0.6$}} & $\textbf{89.6}$ \scriptsize{\raisebox{1pt}{$\pm 0.8$}} \\ 
Medium-Expert & Hopper & $\textbf{110.7}$ \scriptsize{\raisebox{1pt}{$\pm 3.2$}} & $\textbf{111.6}$ \scriptsize{\raisebox{1pt}{$\pm 2.0$}}  \\ 
Medium-Expert & Walker2d & $\textbf{107.7}$ \scriptsize{\raisebox{1pt}{$\pm 0.9$}} & $\textbf{108.2}$ \scriptsize{\raisebox{1pt}{$\pm 0.8$}}\\ 
\midrule
Medium & HalfCheetah & $\textbf{43.8}$ \scriptsize{\raisebox{1pt}{$\pm 0.5$}} & $\textbf{44.2}$ \scriptsize{\raisebox{1pt}{$\pm 0.6$}} \\ 
Medium & Hopper & $95.4$ \scriptsize{\raisebox{1pt}{$\pm 3.4$}} & $\textbf{96.6}$ \scriptsize{\raisebox{1pt}{$\pm 2.7$}} \\ 
Medium & Walker2d & $83.2$ \scriptsize{\raisebox{1pt}{$\pm 3.5$}} & $\textbf{84.4}$ \scriptsize{\raisebox{1pt}{$\pm 2.6$}} \\ 
\midrule
\multicolumn{2}{c}{\textbf{Average}} & 88.4 \hspace{.58cm} & \textbf{89.1} \hspace{.58cm} \\
\bottomrule
\end{tabular}
\vspace{-10pt}
\end{table}

As shown in Table \ref{table:mujoco_iterphases}, with ``Medium" dataset, due to the low quality of the original dataset, although the data generated in the first phase has greatly supplemented the training dataset and greatly improved the performance (referring to Sec \ref{sec:mujoco}), the performance achieved after the second phase is still significantly improved compared with that of the first phase. However, for ``Medium-Expert" dataset, because the expert data of the dataset has covered most of the environment, and the newly generated data is only more suitable for the planner to learn. So, after a certain improvement in the first phase, the subsequent growth is not obvious. The above experiments verify the effectiveness of \alias for the multi-phase iterative paradigm, and also show that the boosting effect is no longer obvious after the algorithm performance reaches a certain level.


\subsubsection{Ablation on Insufficient Data \& Training}
\label{ablation:limit_data}
To demonstrate the superiority of our method over previous diffusion-based work Diffuser~\cite{janner2022planning} when the expert data is limited and the training is insufficient,
we conducted experiments on the Maze2d-Large dataset using different percentages of expert data (e.g. 20\%, 50\%) with only 25\% training steps to train our model.
The results are shown in Table~\ref{tab:ablation_amount}. 
The setting 100\%$\mathcal{D}$ denotes the full training setting.
We can see our \alias, which uses only 50\% data and 25\% training steps, beats the fully trained Diffuser.
\alias can achieve good performance with a small amount of expert data and training steps.

\begin{table}[t]
\caption{\textbf{Ablation study on different amounts of expert data.}}
\vspace{0.2cm}
\centering
\footnotesize
\tabcolsep 8pt
\begin{tabular}{ccccc}
\toprule
   \textbf{Amount of Data} & \textbf{20\% $\mathcal{D}$}   & \textbf{50\% $\mathcal{D}$} & \textbf{100\%$\mathcal{D}$} \\
\midrule
    Diffuser      & 105.0  & 107.9  & 123.0  \\
    AdaptDiffuser & \textbf{112.5}  & \textbf{123.8}  & \textbf{167.9}  \\      
\bottomrule
\end{tabular}
\label{tab:ablation_amount}
\vspace{-5pt}
\end{table}

\subsubsection{Model Size and Running Time}
We show the model size of \alias measured by the number of parameters in Table \ref{table:model_size} here. And we also analyze the testing time and training time performance in Appendix \ref{appendix:time}. From the analysis, we can see that the inference time of \alias is almost equal to that of Diffuser \citep{janner2022planning}.

\begin{table}[t]
\caption{\small \textbf{Model Size of \alias.}}
\label{table:model_size}
\vspace{0.2cm}
\centering
\small
\begin{tabular}{cc}
\toprule
\textbf{Environment} & \textbf{Total Parameters (Model Size)} \\
\midrule
MuJoCo & 3.96 M \\
Maze2D & 3.68 M \\
KUKA Robot & 64.9 M \\
\bottomrule
\end{tabular}
\vspace{-5pt}
\end{table}

\section{Conclusion}

We present \alias, a method for improving the performance of diffusion-based planners in offline reinforcement learning through self-evolution. By generating diverse, high-quality and heterogeneous expert data using a reward-guided diffusion model and filtering out infeasible data using a rule-based discriminator, \alias is able to enhance the performance of diffusion models in existing decision-making tasks, especially the goal-conditioned tasks, and further improve the adaptability in unseen tasks without any expert data. Our experiments on two widely-used offline RL benchmarks and our carefully designed unseen tasks in KUKA and Maze2D environments validate the effectiveness of \alias. 

\textbf{Discussion of Limitation.} Our method achieves better performance by generating high-quality synthetic data but increases the amount of computation required in training with almost no increase in inference time. Besides, although AdaptDiffuser has proven its effectiveness in several scenarios (e.g. MuJoCo, Maze2d, KUKA), it still faces challenges in high-dimensional observation space tasks. More detailed discussions are given in Appendix \ref{appendix:discuss}.

\textbf{Future works.} 
Further improving the sampling speed and exploring tasks with high-dimensional input are potential areas for future works. 
And with the help of ChatGPT \cite{ouyang2022training}, we can use prompts to directly generate diverse maze settings to assist synthetic data generation which is also a promising direction. We provide some examples in Appendix \ref{appendix:chatgpt}.



\section*{Acknowledgements}

This paper is partially supported by the National Key R\&D Program of China No.2022ZD0161000 and the General Research Fund of Hong Kong No.17200622.



\bibliographystyle{icml2023}

\newpage
\appendix
\onecolumn
\section{Classifier-Guided Diffusion Model for Planning}
\label{appedix:classifier}

In this section, we introduce theoretical analysis of conditional diffusion model in detail. We start with an unconditional diffusion probabilistic model with a standard reverse process as $p_{\theta}(\tau^i|\tau^{i+1})$. Then, with a specific label $y$ (for example, goal point in Maze2D or specific reward function in MuJoCo) which is to be conditioned on given a noised trajectory $\tau^i$, the reverse diffusion process can be redefined as $p_{\theta,\phi}(\tau^i|\tau^{i+1},y)$. Apart from the parameters $\theta$ of original diffusion model, a new parameter $\phi$ is introduced here which describes the probability transfer model from noisy trajectory $\tau^i$ to the specific label $y$ which is denoted as $p_{\phi}(y \mid \tau^i)$.

\begin{lemma}
    The marginal probability of a conditional Markov's noising process $q$ conditioned on $y$ is equal to the marginal probability of the unconditional noising process.
\begin{equation}
    q\left(\tau^{i+1} \mid \tau^{i}\right) = q\left(\tau^{i+1} \mid \tau^{i}, y\right)
\end{equation}
\end{lemma}

\begin{proof}
\begin{equation*}
    \begin{split}
    q\left(\tau^{i+1} \mid \tau^{i}\right) & =\int_{y} q\left(\tau^{i+1}, y \mid \tau^{i}\right) d y 
    \\
    & =\int_{y} q\left(\tau^{i+1} \mid \tau^{i}, y\right) p_{\phi}\left(y \mid \tau^{i}\right) d y 
    \\
    & =q\left(\tau^{i+1} \mid \tau^{i}, y\right) \int_{y} p_{\phi}\left(y \mid \tau^{i}\right) d y 
    \\
    & =q\left(\tau^{i+1} \mid \tau^{i}, y\right)
\end{split}
\end{equation*}
The third line holds because $q\left(\tau^{i+1} \mid \tau^{i}, y\right)$ fits another $y$-independent transition probability according to its definition.
\end{proof}

\begin{lemma}
The probability distribution of specific label $y$ conditioned on $\tau^i$ does not depend on $\tau^{i+1}$.
\begin{equation}
    p_{\theta, \phi}\left(y \mid \tau^{i}, \tau^{i+1}\right) = p_{\phi}\left(y \mid \tau^{i}\right)
\end{equation}
\end{lemma}

\begin{proof}
\begin{equation*}
    \begin{split}
    p_{\theta, \phi}\left(y \mid \tau^{i}, \tau^{i+1}\right) &=q\left(\tau^{i+1} \mid \tau^{i}, y\right) \frac{p_{\phi}\left(y \mid \tau^{i}\right)}{q\left(\tau^{i+1} \mid \tau^{i}\right)} 
    \\
    & =q\left(\tau^{i+1} \mid \tau^{i}\right) \frac{p_{\phi}\left(y \mid \tau^{i}\right)}{q\left(\tau^{i+1} \mid \tau^{i}\right)} 
    \\
    & =p_{\phi}\left(y \mid \tau^{i}\right)
\end{split}
\end{equation*}
\end{proof}

\begin{theorem}
The conditional sampling probability $p_{\theta,\phi}(\tau^i \mid \tau^{i+1},y)$ is proportional to unconditional transition probability $p_{\theta}(\tau^i \mid \tau^{i+1})$ multiplied by classified probability $p_{\phi}(y \mid \tau^i)$.
\begin{equation}
    p_{\theta,\phi}(\tau^i \mid \tau^{i+1},y) = Z p_{\theta}(\tau^i \mid \tau^{i+1})p_{\phi}(y \mid \tau^i)
\label{eq:theorem1}
\end{equation}
\end{theorem}

\begin{proof}
\begin{equation}
    \begin{split}
p_{\theta,\phi}(\tau^i \mid \tau^{i+1},y) &= \frac{p_{\theta, \phi}\left(\tau^{i}, \tau^{i+1}, y\right)}{p_{\theta, \phi}\left(\tau^{i+1}, y\right)} 
\\
& =\frac{p_{\theta, \phi}\left(\tau^{i}, \tau^{i+1}, y\right)}{p_{\phi}\left(y \mid \tau^{i+1}\right) p_{\theta}\left(\tau^{i+1}\right)} 
\\
& =\frac{p_{\theta}\left(\tau^{i} \mid \tau^{i+1}\right) p_{\theta, \phi}\left(y \mid \tau^{i}, \tau^{i+1}\right) p_{\theta}\left(\tau^{i+1}\right)}{p_{\phi}\left(y \mid \tau^{i+1}\right) p_{\theta}\left(\tau^{i+1}\right)} 
\\
& =\frac{p_{\theta}\left(\tau^{i} \mid \tau^{i+1}\right) p_{\theta, \phi}\left(y \mid \tau^{i}, \tau^{i+1}\right)}{p_{\phi}\left(y \mid \tau^{i+1}\right)} 
\\
& =\frac{p_{\theta}\left(\tau^{i} \mid \tau^{i+1}\right) p_{\phi}\left(y \mid \tau^{i}\right)}{p_{\phi}\left(y \mid \tau^{i+1}\right)} 
\end{split}
\end{equation}

The term $p_{\phi}\left(y \mid \tau^{i+1}\right)$ can be seen as a constant since it's not conditioned on $\tau^{i}$ at the diffusion timestep $i$.
\end{proof}


Although exact sampling from this distribution (Equation \ref{eq:theorem1}) is difficult, \cite{sohl2015deep} demonstrates that it can be approximated as a modified Gaussian distribution. We show the derivation here.

On one hand, as Equation \ref{eq:p1} shows, we can formulate the denoising process with a Gaussian distribution:
\begin{align}
    p_{\theta}(\tau^i \mid \tau^{i+1}) &= \mathcal{N}(\mu, \Sigma) \\
    \log p_{\theta}(\tau^i \mid \tau^{i+1}) &= -\frac{1}{2}(\tau^i - \mu)^T \Sigma^{-1} (\tau^{i} - \mu) + C
    \label{eq:tt1}
\end{align}

And on the other hand, the number of diffusion steps are usually large, so the difference between $\tau^i$ and $\tau^{i+1}$ is small enough. We can apply Taylor expansion around $\tau^i=\mu$ to $\log p_{\phi}(y \mid \tau^i)$ as,
\begin{equation}
    \log p_{\phi}\left(y \mid \tau^{i}\right) = \log p_{\phi}\left(y \mid \tau^{i}\right)|_{\tau^{i}=\mu}+\left.\left(\tau^{i}-\mu\right) \nabla_{\tau^{i}} \log p_{\phi}\left(y \mid \tau^{i}\right)\right|_{\tau^{i}=\mu}
    \label{eq:tt2}
\end{equation}

Therefore, synthesize Equation \ref{eq:tt1} and \ref{eq:tt2}, we derive,
\begin{equation}
    \begin{split}
        \log p_{\theta,\phi}(\tau^i|\tau^{i+1},y) &= \log p_{\theta}(\tau^i|\tau^{i+1}) + \log p_{\phi}(y|\tau^i)+C_1
        \\
         & = -\frac{1}{2}\left(\tau^{i}-\mu\right)^{T} \Sigma^{-1}\left(\tau^{i}-\mu\right)+\left(\tau^{i}-\mu\right) \nabla \log p_{\phi}\left(y \mid \tau^{i}\right) +C_{2}
        \\
        & =-\frac{1}{2}\left(\tau^{i}-\mu-\Sigma \nabla \log p_{\phi}\left(y \mid \tau^{i}\right)\right)^{T} \Sigma^{-1}\left(\tau^{i}-\mu-\Sigma \nabla \log p_{\phi}\left(y \mid \tau^{i}\right)\right)+C_{3} 
    \end{split}
\end{equation}
which means,
\begin{equation}
    p_{\theta,\phi}(\tau^i|\tau^{i+1},y) \approx \mathcal{N}(\tau^i; \mu + \Sigma \nabla_{\tau} \log p_{\phi}\left(y \mid \tau^{i}\right), \Sigma)
\end{equation}
And it's equal to Equation \ref{eq:guided}. Proven.

\newpage
\section{Visualization Results of KUKA Pick-and-Place Task}
\label{appedix:vis_kuka}

In this section, we show more visualization results about KUKA pick-and-place task. We require the KUKA Robot Arm to pick green, yellow, blue and red blocks with random initialized positions on the right side of the table one by one and move them to the left side in the order of yellow, blue, green and red (from near to far).

\subsection{Pick and Place 1st Green Block}
\begin{figure}[!htb]
    \begin{center}
    \centerline{\includegraphics[width=\linewidth]{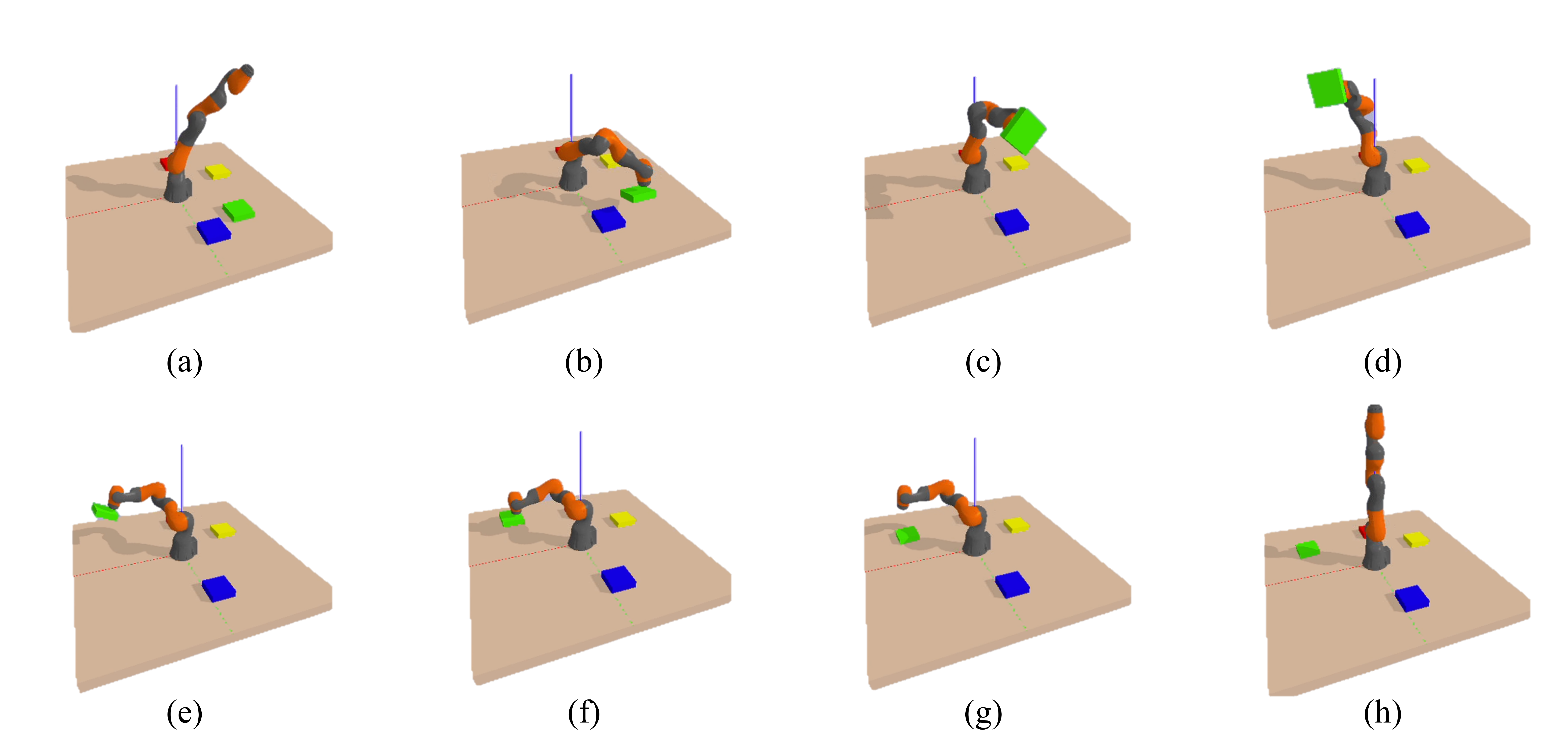}}
    \vspace{-0.15in}
    \caption{The Process of Pick and Place Block 1 (Green Block)}
    \label{fig:kuka_line1}
    \end{center}
    \vspace{-0.2in}
\end{figure}

\subsection{Pick and Place 2nd Yellow Block}
\begin{figure}[!htb]
    \begin{center}
    \centerline{\includegraphics[width=\linewidth]{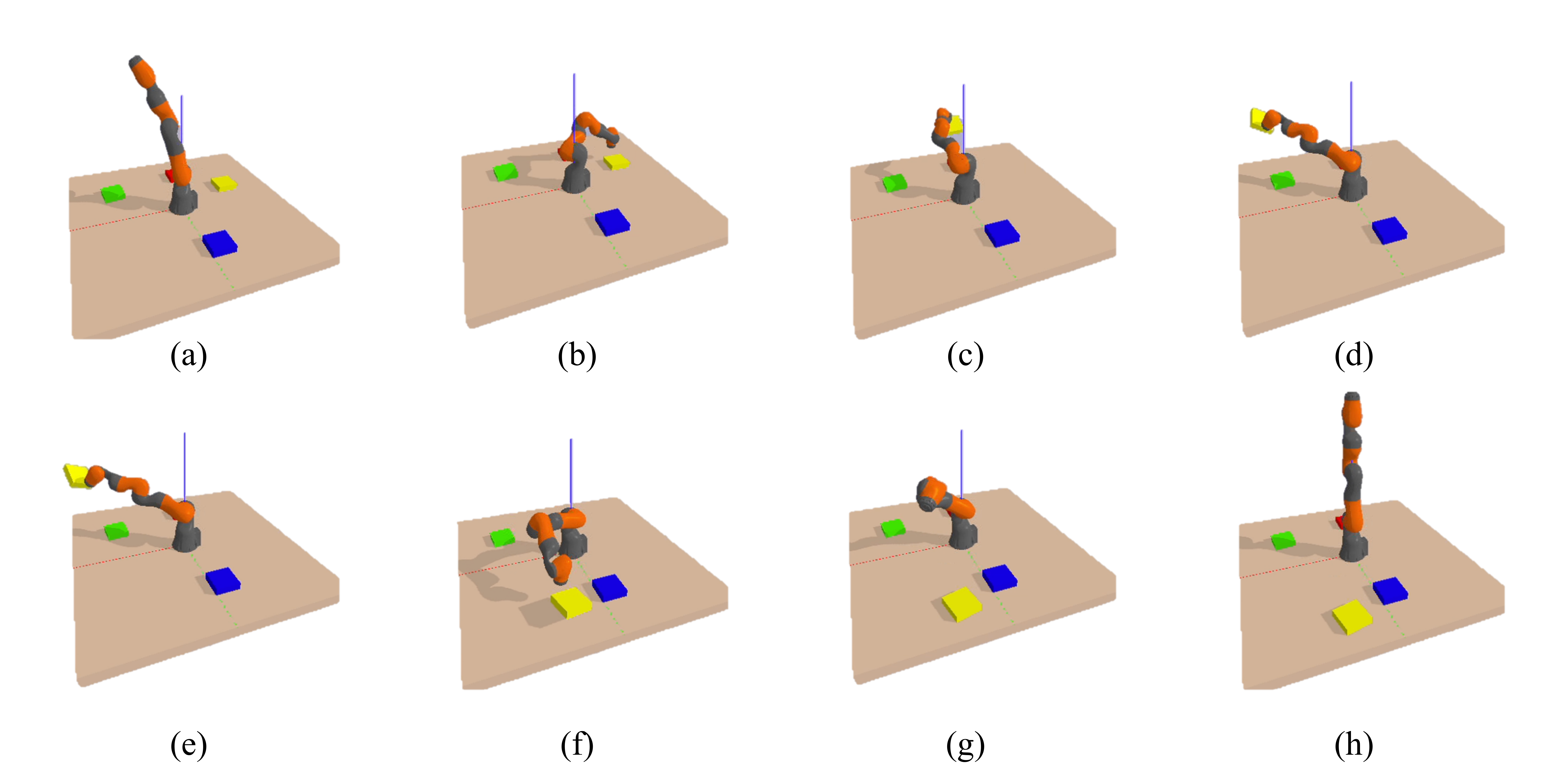}}
    \vspace{-0.15in}
    \caption{The Process of Pick and Place Block 2 (Yellow Block)}
    \label{fig:kuka_line2}
    \end{center}
    \vspace{-0.2in}
\end{figure}

\subsection{Pick and Place 3rd Blue Block}
\begin{figure}[!htb]
    \begin{center}
    \centerline{\includegraphics[width=\linewidth]{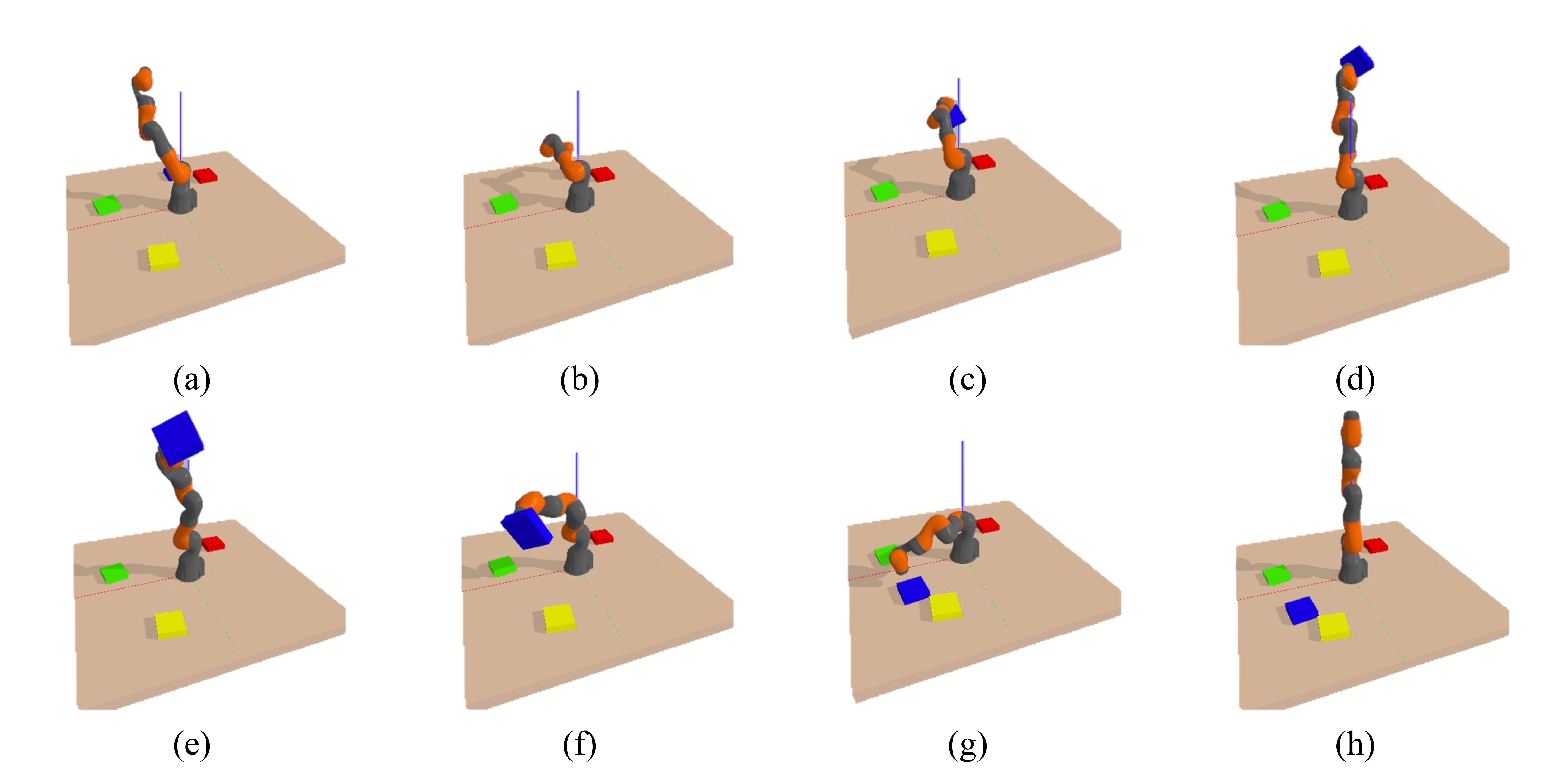}}
    \vspace{-0.15in}
    \caption{The Process of Pick and Place Block 3 (Blue Block)}
    \label{fig:kuka_line3}
    \end{center}
    \vspace{-0.2in}
\end{figure}

\subsection{Pick and Place 4th Red Block}
\begin{figure}[!htb]
    \begin{center}
    \centerline{\includegraphics[width=\linewidth]{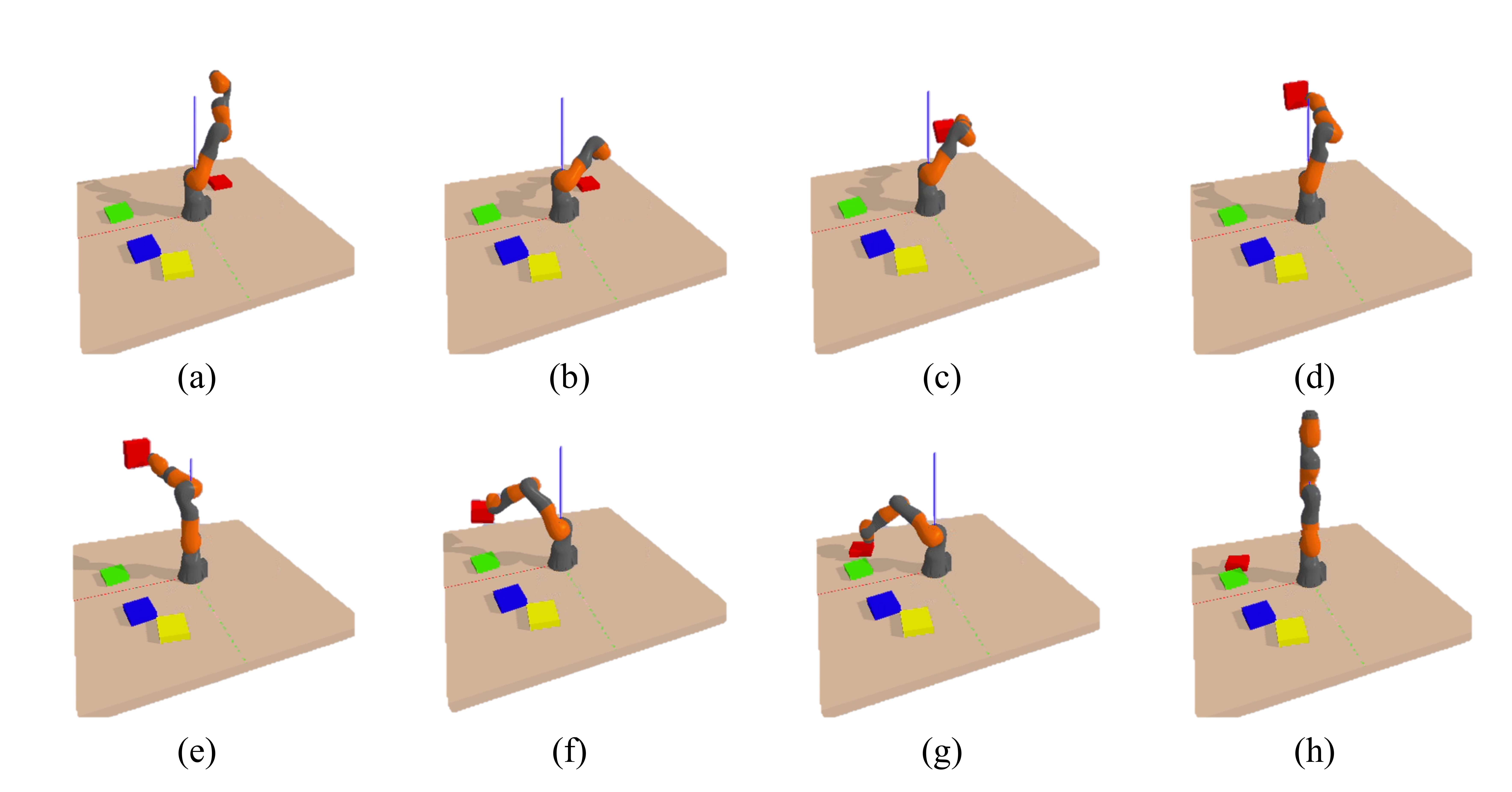}}
    \vspace{-0.15in}
    \caption{The Process of Pick and Place Block 4 (Red Block)}
    \label{fig:kuka_line4}
    \end{center}
    \vspace{-0.2in}
\end{figure}


\newpage
\section{Implementation Details and Hyperparameters}
\label{appedix:detail}

\subsection{Details of Baseline Performances}

\textbf{Maze2D Tasks.} 
We perform two different tasks on the Maze2D environment to validate the performance enhancement and adaptation ability of \alias on seen and unseen tasks.
\begin{itemize}
    \item \textbf{Overall Performance of Navigation Task}: We report the performance of CQL and IQL on the standard Maze2D environments from Table~2 in D4RL whitepaper \citep{fu2020d4rl} and follow the hyperparameter settings described in \cite{janner2022planning}. The performance of Diffuser also refers to Table~1 in \cite{janner2022planning}. To reproduce the experimental results, we use the official implementation from the authors of IQL\footnote{\scriptsize\url{https://github.com/ikostrikov/implicit_q_learning} \label{diffuser_code}} and Diffuser\footnote{\scriptsize\url{https://github.com/jannerm/diffuser}}.

    \item \textbf{Navigation with Gold Coin Picking Task}: We modified the official code of Diffuser and tuned over the hyperparameter $\alpha \in \{-50, -100, -200\}$ (the scalar of the guidance) in Equation \ref{eq:cond_gen_model} to adjust the planner to be competent for newly designed gold coin picking task, which is also the basis of our method \alias.
    
\end{itemize}

\textbf{KUKA Pick and Place Tasks.}
Similar to the unseen tasks in Maze2D environment, we also ran the official implementation of IQL and Diffuser.

\textbf{MuJoCo Locomotion Tasks.} 
We report the scores of BC, CQL and IQL from Table~1 in \cite{kostrikov2021offline}. We take down scores of DT from Table~2 in \cite{chen2021decision}, TT from Table~1 in \cite{janner2021offline}, MOPO from Table~1 in \cite{yu2020mopo}, MOReL from Table~2 in \cite{kidambi2020morel}, MBOP from Table~1 in \cite{argenson2020model} and Diffuser from Table~2 in \cite{janner2022planning}. All baselines are trained using the same offline dataset collected by a specific expert policy.

\begin{table*}[h]
\caption{\small \textbf{Metric Values for Reward Discriminator in MuJoCo Environment.} 
The rewards are calculated utilizing D4RL \cite{fu2020d4rl} locomotion suite.
}
\label{table:mujoco_setting}
\vspace{0.2cm}
\centering
\small
\begin{tabular}{cccc}
\toprule
\textbf{Dataset} & \textbf{Environment} & \textbf{1$^{\text{st}}$ Phase} & \textbf{2$^{\text{nd}}$ Phase}\\ 
\midrule
Med-Expert & HalfCheetah & 10840 & 10867 \\ 
Med-Expert & Hopper & 3639 & 3681 \\ 
Med-Expert & Walker2d & 4900 & 4950 \\ 
\midrule
Medium & HalfCheetah & 5005 & 5150 \\ 
Medium & Hopper & 3211 & 3225
 \\ 
Medium & Walker2d & 3700 & 3843 \\ 
\midrule
Med-Replay & HalfCheetah & 4600 & 4800
\\ 
Med-Replay & Hopper & 3100 & 3136 \\ 
Med-Replay & Walker2d & 3900 & 3920 \\ 
\bottomrule
\end{tabular}
\end{table*}

\subsection{Metric Values for Reward Discriminator}
\label{appendix:metric_value}
\textbf{Maze2D Environment}. For the three different-size Maze2D settings, unlike MuJoCo, different trajectories are different in lengths which achieve different rewards. So, we not only consider the absolute value of the rewards $\mathcal{R}$ but also introduce trajectory length $\mathcal{L}$ and reward-length ratio into the criteria of discrimination. We prefer trajectories with longer lengths or those having higher reward-length ratios. Additionally, we denote the maximum episode steps of the environment as $Max_e$ (Maze2D-UMaze: $300$, Maze2D-Medium: $600$, Maze2D-Large: $800$). And then, we have following metrics to filter out high-quality data.

\begin{itemize}
    \item \textbf{Maze2D-UMaze}: The trajectory is required to satisfy $\mathcal{L} > 200$ or $\mathcal{L} > 50$ and $\mathcal{R} + 1.0 * (Max_e - \mathcal{L}) > 210$ which is equal to measure the $\mathcal{R}/\mathcal{L}$.

    \item \textbf{Maze2D-Medium}: The trajectory is required to satisfy $\mathcal{L} > 450$ or $\mathcal{L} > 200$ and $\mathcal{R} + 1.0 * (Max_e - \mathcal{L}) > 400$.

    \item \textbf{Maze2D-Large}: The trajectory is required to satisfy $\mathcal{L} > 650$ or $\mathcal{L} > 270$ and $\mathcal{R} + 1.0 * (Max_e - \mathcal{L}) > 400$.
    
\end{itemize}

\textbf{KUKA Robot Arm}. For the KUKA Robot Arm environment, we define a sparse reward function that achieves one if and only if the placement is successful and zero otherwise. Therefore, we take the condition $\mathcal{R} >= 2.0$ which means at least half of the four placements are successful. 

\textbf{MuJoCo Environment.}
For MuJoCo locomotion environment, as we describe in Sec. \ref{sec:mujoco}, we directly use the reward derived after generated state sequence and action sequence to filter out high-quality synthetic data. The specific values for MuJoCo are shown in Table \ref{table:mujoco_setting}.

\subsection{Amount of Synthetic Data for Each Iteration}

The amount of synthetic data for each iteration is another important hyperparameter for \alias. Different tasks have different settings. We give detailed hyperparameters here.

\begin{table*}[htb]
\caption{\small \textbf{Amount of Synthetic Data for Each Iteration.} The number of synthetic data for KUKA Arm pick-and place task consists of 1000 generated trajectories and 10000 cross-domain trajectories from the unconditional stacking task.
}
\label{table:amount_data}
\vspace{0.2cm}
\centering
\small
\begin{tabular}{cccc}
\toprule
\textbf{Dataset} & \textbf{Task} & \textbf{\# of Expert Data} & \textbf{\# of Synthetic Data} \\ 
\midrule
MuJoCo & Locomotion & $10^6$, $2 \times 10^6$& 50000 \\
\midrule
Maze2D & Navigation & $10^6$, $2 \times 10^6$, $4 \times 10^6$ & $10^6$ \\
Maze2D & Gold Coin Picking & 0 & $10^6$\\
\midrule
KUKA Robot & Unconditional Stacking & 10000 & - \\
KUKA Robot & Pick-and-Place & 0 & 11000 \\ 
\bottomrule
\end{tabular}
\end{table*}


\subsection{Other Details}
\begin{enumerate}
\item A temporal U-Net~\cite{ronneberger2015u} with 6 repeated residual blocks is employed to model the noise $\epsilon_\theta$ of the diffusion process. Each block is comprised of two temporal convolutions, each followed by group norm \cite{wu2018group}, and a final Mish non-linearity \cite{misra2019mish}. Timestep embeddings are generated by a single fully-connected layer and added to the activation output after the first temporal convolution of each block.

\item The diffusion model is trained using the Adam optimizer \citep{kingma2014adam} with a learning rate of $2\times10^{-4}$ and batch size of $32$.

\item The training steps of the diffusion model are $1M$ for MuJoCo locomotion task, $2M$ for tasks on Maze2D and $0.7M$ for KUKA Robot Arm tasks.

\item The planning horizon $T$ is set as 32 in all locomotion tasks, $128$ for KUKA pick-and-place, $128$ in Maze2D-UMaze, $192$ in Maze2D-Medium, and $384$ in Maze2D-Large.

\item We use $K = 100$ diffusion steps for all locomotion tasks, $1000$ for KUKA robot arm tasks, $64$ for Maze2D-UMaze, $128$ for Maze2D-Medium, and $256$ for Maze2D-Large.

\item We choose 2-norm as the auxiliary guided function in the combination setting of Section \ref{sec:method_gen} and the guidance scale $\alpha \in \{1, 5, 10, 50, 100\}$ of which the exact choice depends on the specific task.

\end{enumerate}

\newpage
\section{Testing-time and Training-time Analysis}
\label{appendix:time}
\subsection{Testing-time Characteristic of \alias}
\alias only generates synthetic data during training and performs denoising once during inference to obtain the optimal trajectory. We show the inference time of generating an action taken by Diffuser \citep{janner2022planning} and our method in Table \ref{table:test_time_mujoco} and Table \ref{table:test_time_maze}. All these data are tested with one \textit{NVIDIA RTX 3090 GPU}.

\begin{table*}[htb]
\vspace{-10pt}
\caption{\small \textbf{Testing Time in D4RL MuJoCo Environment.} The unit in the table is second (s).
}
\label{table:test_time_mujoco}
\vspace{0.2cm}
\centering
\small
\begin{tabular}{cccc}
\toprule
\textbf{Dataset} & \textbf{Environment} & \textbf{Diffuser} & \textbf{AdaptDiffuser} \\ 
\midrule
Med-Expert & HalfCheetah & 1.38 s & 1.41 s \\ 
Med-Expert & Hopper & 1.57 s & 1.59 s \\ 
Med-Expert & Walker2d & 1.60 s & 1.56 s \\ 
\midrule
Medium & HalfCheetah & 1.40 s & 1.40 s \\ 
Medium & Hopper & 1.60 s & 1.56 s \\ 
Medium & Walker2d & 1.57 s & 1.57 s \\ 
\midrule
Med-Replay & HalfCheetah & 1.43 s & 1.37 s
\\ 
Med-Replay & Hopper & 1.59 s & 1.55 s \\ 
Med-Replay & Walker2d & 1.55 s & 1.58 s \\ 
\bottomrule
\end{tabular}
\vspace{-8pt}
\end{table*}

\begin{table*}[htb]
\caption{\small \textbf{Testing Time in D4RL Maze2D and KUKA Environments.} The test time of KUKA is derived by dividing the trajectory generation time by horizon size. The unit in the table is second (s).}
\label{table:test_time_maze}
\vspace{0.2cm}
\centering
\small
\begin{tabular}{ccc}
\toprule
\textbf{Environment} & \textbf{Diffuser} & \textbf{AdaptDiffuser} \\ 
\midrule
Maze2D U-Maze & 0.70 s & 0.69 s \\
Maze2D Medium & 1.42 s & 1.44 s \\
Maze2D Large & 2.80 s & 2.76 s \\
\midrule
KUKA Pick and Place & 0.21 s & 0.21 s \\
\bottomrule
\end{tabular}
\end{table*}

From the tables, we can see that the inference time of AdaptDiffuser is almost equal to that of Diffuser \citep{janner2022planning}. And because the denoising steps of different datasets are different, the testing times are different between environments. For MuJoCo, the inference time of an action is approximately 1.5s, while for Maze2D the inference time is about 1.6s (on average of three environments), and for KUKA about 0.21s. The inference time is feasible for real-time robot control. Additionally, in Section \ref{ablation:limit_data} of our paper, we have also demonstrated how limited number of high quality expert data would affect our method's performance.

What's more, as suggested in Diffuser \citep{janner2022planning}, we can improve the testing time by warm-starting the state diffusion, which means we start with the state sequence generated from the previous environment step and then reduce the number of denoising steps.

\begin{table*}[b]
\vspace{-8pt}
\caption{\small \textbf{Synthetic Data Generation Time and Training Time in MuJoCo Environment.} The synthetic data generation time listed here is about the time to generate one high-quality trajectory. The total training time of \alias is the sum of the following three parts. The quality standard of selected trajectories are the same as those stated in Appendix \ref{appendix:metric_value}. The unit in the table is hour (h).}
\label{table:train_time_mujoco}
\vspace{0.2cm}
\centering
\small
\begin{tabular}{ccccc}
\toprule
\textbf{Dataset} & \textbf{Environment} & \textbf{Synthetic Data Gen. Time} & \textbf{AdaptDiffuser Fine-Tuning} & \textbf{Diffuser Training} \\ 
\midrule
Med-Expert & HalfCheetah & 4.4 h & 6.8 h & 44.2 h \\ 
Med-Expert & Hopper & 5.7 h & 6.4 h & 37.0 h \\ 
Med-Expert & Walker2d & 3.0 h & 6.6 h & 43.0 h \\ 
\midrule
Medium & HalfCheetah & 2.4 h & 7.0 h & 45.3 h \\ 
Medium & Hopper & 4.8 h & 6.2 h & 36.2 h \\ 
Medium & Walker2d & 4.7 h & 6.4 h & 43.0 h \\ 
\midrule
Med-Replay & HalfCheetah & 15.7 h & 7.4 h & 45.3 h
\\ 
Med-Replay & Hopper & 11.9 h & 6.5 h & 36.1 h \\ 
Med-Replay & Walker2d & 4.3 h & 6.4 h & 42.8 h \\ 
\bottomrule
\end{tabular}
\end{table*}

\subsection{Training-time Characteristic of \alias}
The training time of \alias can be seen as the sum of synthetic data generation time and diffusion model training time. The synthetic data generation time depends on the quality standard of the trajectory to be selected.

What's more, to accelerate the training, we use the warming-up technique which takes the pre-trained Diffuser model as the basis of AdaptDiffuser, and then performs fine-tuning on new generated data with fewer training steps (1/4 in actual use). Then we show these three parts' times in Table \ref{table:train_time_mujoco}. All these times are tested with one \textit{NVIDIA RTX 3090 GPU}.

It can be found from the table that the model training time dominates the total pre-training time while the extra time spent, such as synthetic data generation, is a relatively small part. The total time required to pre-train AdaptDiffuser is on average 54 hours (sum of the three parts) comparable to Diffuser's 41 hours.

Besides, the data generation process can be executed parallel. For example, in our D4RL MuJoCo environment, we generate 10 trajectories for each dataset at each phase. Under parallel settings, the total time to collect all ten synthetic trajectories is the same as the time to collect one trajectory. If using more GPUs, the synthetic data generation time can be further reduced.

\section{Comparison with Decision Diffuser}
Decision Diffuser (DD) \citep{ajay2022conditional} is a concurrent work with ours and improves the performance of Diffuser \citep{janner2022planning} by introducing planning with classifier-free guidance and acting with inverse-dynamics.

Generally speaking, our method is a general algorithm that enables diffusion-based planners to have self-evolving ability that can perform well on existing and unseen (zero-shot) tasks, mainly by generating high-quality synthetic data with reward and dynamics consistency guidance for diverse tasks simultaneously. Therefore, regardless of which diffusion-based planner to be used, there can exist AdaptDiffuser, AdaptDecisionDiffuser, etc. It means that the method we introduce to make the planner self-evolving does not conflict with the improvements proposed by Decision Diffuser. The improvements of these two works can complement each other to further enhance the performance of diffusion model-based planners.

We also compare the performance of Decision Transformer (DT) \citep{chen2021decision}, Trajectory Transformer (TT) \citep{janner2021offline}, Diffuser \citep{janner2022planning}, Decision Diffuser \citep{ajay2022conditional} and our method here. Results about Decision Diffuser are quoted from \cite{ajay2022conditional}.

\begin{table*}[htb]
\vspace{-5pt}
\caption{\small \textbf{Performance Comparison with Decision Diffuser in MuJoCo Environment.} 
We report normalized average returns of D4RL tasks \citep{fu2020d4rl} in the table. And the mean and the standard error are calculated over 3 random seeds.
}
\label{table:with_DD}
\vspace{0.2cm}
\centering
\small
\tabcolsep 4.5pt
\begin{tabular}{ccccccc}
\toprule
\textbf{Dataset} & \textbf{Environment} & \textbf{DT} & \textbf{TT} & \textbf{Diffuser} & \textbf{Decision Diffuser} & \textbf{AdaptDiffuser}\\ 
\midrule
Med-Expert & HalfCheetah & $86.8$ & $\textbf{95.0}$ & $88.9$ & $90.6$ & $89.6$ \scriptsize{\raisebox{1pt}{$\pm 0.8$}} \\
Med-Expert & Hopper & $107.6$ & $110.0$ & $103.3$ & $\textbf{111.8}$ & $\textbf{111.6}$ \scriptsize{\raisebox{1pt}{$\pm 2.0$}} \\
Med-Expert & Walker2d & $\textbf{108.1}$ & $101.9$ & $106.9$ & $\textbf{108.8}$ & $\textbf{108.2}$ \scriptsize{\raisebox{1pt}{$\pm 0.8$}} \\
\midrule
Medium & HalfCheetah & $42.6$ & $46.9$ & $42.8$ & $\textbf{49.1}$ & $44.2$ \scriptsize{\raisebox{1pt}{$\pm 0.6$}} \\
Medium & Hopper & $67.6$ & $61.1$ & $74.3$ & $79.3$ & $\textbf{96.6}$ \scriptsize{\raisebox{1pt}{$\pm 2.7$}} \\
Medium & Walker2d & $74.0$ & $79.0$ & $79.6$ & $82.5$ & $\textbf{84.4}$ \scriptsize{\raisebox{1pt}{$\pm 2.6$}} \\
\midrule
Med-Replay & HalfCheetah & $36.6$ & $41.9$ & $37.7$ & $\textbf{39.3}$ & $\textbf{38.3}$ \scriptsize{\raisebox{1pt}{$\pm 0.9$}} \\
Med-Replay & Hopper & $82.7$ & $91.5$ & $93.6$ & $\textbf{100.0}$ & $92.2$ \scriptsize{\raisebox{1pt}{$\pm 1.5$}} \\
Med-Replay & Walker2d & $66.6$ & $82.6$ & $70.6$ & $75.0$ & $\textbf{84.7}$ \scriptsize{\raisebox{1pt}{$\pm 3.1$}} \\
\midrule
\multicolumn{2}{c}{\textbf{Average}} & 74.7 & 78.9 & 77.5 & 81.8 & \textbf{83.4} \hspace{.58cm} \\
\bottomrule
\end{tabular}
\end{table*}

From the table, we can see that in most datasets, the performance of \alias is comparable to or better than that of Decision Diffuser. And the normalized average return of \alias is $83.4$ higher than all of the other methods (i.e. $74.7$ of DT, $78.9$ of TT, $77.5$ of Diffuser and $81.8$ of Decision Diffuser).

\newpage
\section{Discussions}
\label{appendix:discuss}
\subsection{Adapt \alias to Maze2D Gold Coin Picking Task with Coin Locating Far from the Optimal Path}

\alias works when the gold coin is located nowhere near the optimal path. Figure \ref{fig:maze_adapt} of our paper has shown one case. The sub-figure (b) of Figure \ref{fig:maze_adapt} show the optimal path when there are no gold coins in the maze. (The generated route walks at the bottom of the figure.) And then if we add a gold coin in the (4,2) position of the maze, \alias will generate a new path that passes through the gold coin as shown in the sub-figure (d) of Figure \ref{fig:maze_adapt}. (The generated route walks in the middle of the figure.)

In our point of view, our method works mainly because we change the start point and goal point multiple times during training. Diffusion model can generate trajectories that have not been seen in the expert dataset. And as long as the paths generated during training can cover the entire trajectory space as much as possible, \alias can generate the path through any location of the gold coin during planning. However, it is true that the success rate of generating trajectories for some extremely hard cases that the gold coin is far from the planned path and the agent has to take a turn back to obtain the gold coin, is lower than that of common cases.

\subsection{Adapt \alias to High-dimensional Observation Space Tasks}
\alias is feasible for high-dimensional observation space tasks. One possible and widely-used solution, we suggest, is to add an embedding module (e.g. MLP) after input to convert the data from high-dimensional space to latent space, and then employ \alias in latent space to solve the problem. Stable Diffusion \cite{rombach2022high} has shown the effectiveness of this method, which deploys an Auto-Encoder to encode image into a latent representation and uses a decoder to reconstruct the image from the latent after denoising. MineDoJo \cite{fan2022minedojo} also takes this technique and achieves outstanding performance in image-based RL domain.

\newpage
\section{Generate Diverse Maze Layouts with ChatGPT}
\label{appendix:chatgpt}
Inspired by the remarkable generation capabilities demonstrated by recent advancements in large language models (LLMs), exemplified by ChatGPT, we propose a novel approach that harnesses the potential of LLM to accelerate the process of synthetic data generation. In this section, we focus specifically on utilizing LLM to assist in generating diverse Maze layouts. This objective is driven by the need to create a multitude of distinct maze layouts to facilitate varied path generations, ultimately enhancing the performance and adaptability of \alias.
Traditionally, the manual design of feasible and terrain complex maze environments
is a time-consuming endeavor that requires to try and adjust multiple times. In light of this challenge, leveraging ChatGPT for maze environment generation emerges as an appealing alternative, streamlining the process and offering enticing advantages.
We show the generated examples in Fig. \ref{fig:chatgpt_maze}. Besides, we can ask the ChatGPT to summarize the rules of generating feasible mazes, shown in Fig. \ref{fig:rule_chat}.

\begin{figure}[htb]
    \vspace{-0.1in}
    \begin{center}
    \centerline{\includegraphics[width=0.7\linewidth]{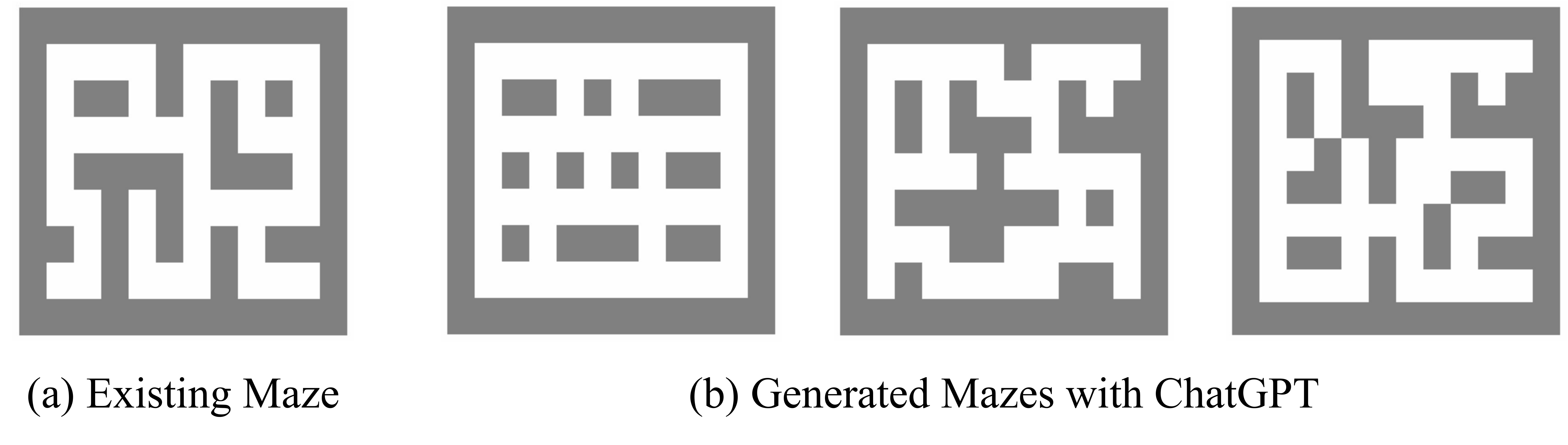}}
    \vspace{-0.2in}
    \caption{\textbf{Generated Maze examples by ChatGPT.} From simple terrain to complex terrain (with multiple dead ends and loops).}
    \label{fig:chatgpt_maze}
    \end{center}
    \vspace{-0.3in}
\end{figure}

\begin{figure}[htb]
    \begin{center}
    \centerline{\includegraphics[width=0.7\linewidth]{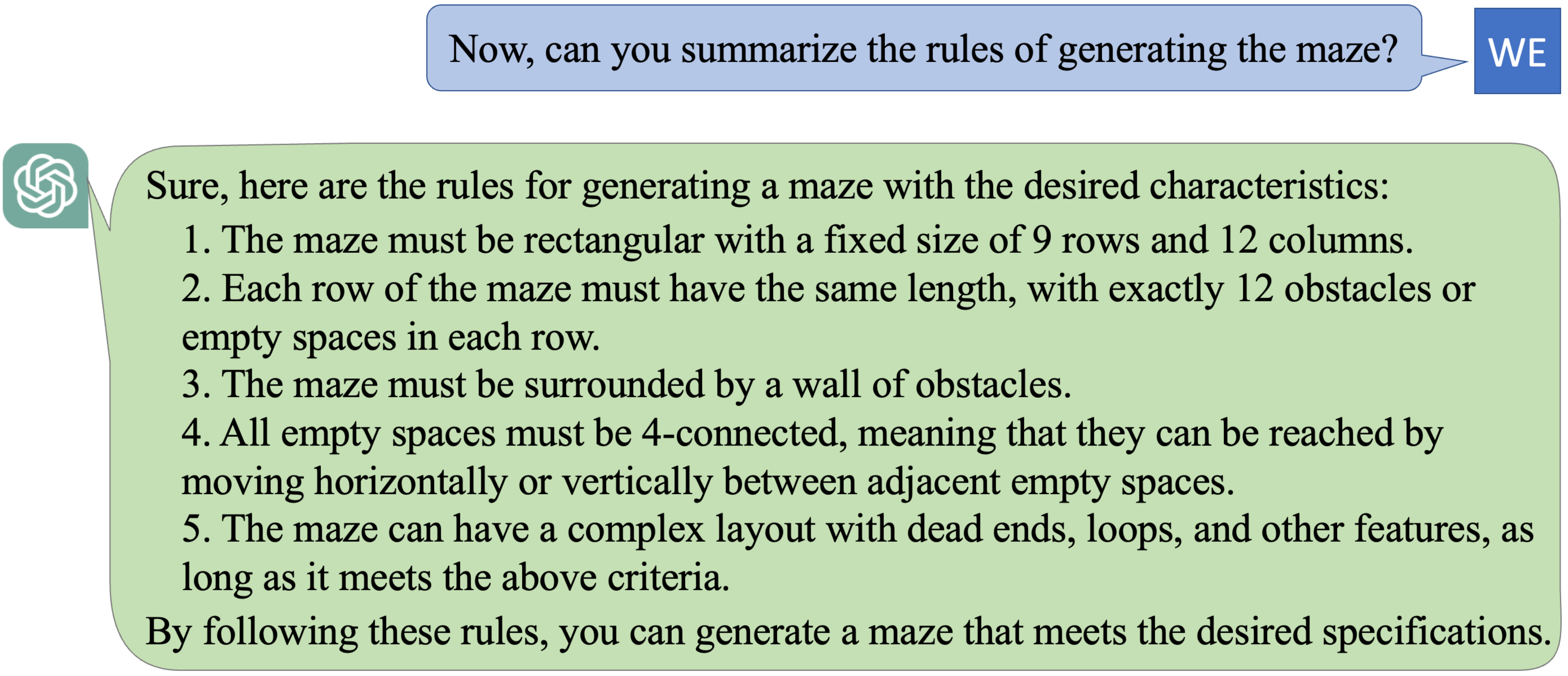}}
    \vspace{-0.1in}
    \caption{\textbf{Rules for generating maze layouts summarized by ChatGPT.}}
    \label{fig:rule_chat}
    \end{center}
    \vspace{-0.3in}
\end{figure}

We also give our prompts here. We find that providing ChatGPT with a few existing feasible maze examples (few-shot) can effectively improve the quality of the generated mazes, so we design the prompts in this way. From prompt 1 to prompt 2, we also find that the terrains of generated mazes are exactly from simple to complex.

\textbf{Prompt1:} ``I will give you a legal string expression of a MAZE. In the MAZE, the `\#' represents the obstacles and the `O' represents the empty space. Could you generate one more maze with different terrain obeying to the rules: The MAZE should be 9*12, and the surrounding of the MAZE should be obstacles, that is `\#', and all empty places should be 4-connected. The example maze is 
\begin{equation*}
\begin{aligned}
\text{LARGE\_MAZE} = 
    &``\#\#\#\#\#\#\#\#\#\#\#\#\backslash\backslash"+\\
    &``\#OOOO\#OOOOO\#\backslash\backslash"+\\
    &``\#O\#\#O\#O\#O\#O\#\backslash\backslash"+\\
    &``\#OOOOOO\#OOO\#\backslash\backslash"+\\
    &``\#O\#\#\#\#O\#\#\#O\#\backslash\backslash"+\\
    &``\#OO\#O\#OOOOO\#\backslash\backslash"+\\
    &``\#\#O\#O\#O\#O\#\#\#\backslash\backslash"+\\
    &``\#OO\#OOO\#OOO\#\backslash\backslash"+\\
    &``\#\#\#\#\#\#\#\#\#\#\#\#"\ "
\end{aligned}
\end{equation*}

\textbf{Prompt2:} ``Please generate more complex Maze that has more complex terrains (i.e. more dead ends, loops, and obstacles)".

\end{document}